\documentclass{article}

     \usepackage[nonatbib,preprint]{neurips_2020}

\usepackage[utf8]{inputenc} %
\usepackage[T1]{fontenc}    %
\usepackage{hyperref}       %
\usepackage{url}            %
\usepackage{booktabs}       %
\usepackage{amsfonts}       %
\usepackage{nicefrac}       %
\usepackage{microtype}      %

\usepackage{amsmath,amsfonts,bm}

\def\eqref#1{equation~\ref{#1}}

\def\1{\bm{1}}

\def\rvw{{\mathbf{w}}}
\def\rvx{{\mathbf{x}}}

\def\rvz{{\mathbf{z}}}

\DeclareMathAlphabet{\mathsfit}{\encodingdefault}{\sfdefault}{m}{sl}
\SetMathAlphabet{\mathsfit}{bold}{\encodingdefault}{\sfdefault}{bx}{n}

\def\gL{{\mathcal{L}}}

\def\gU{{\mathcal{U}}}

\def\sD{{\mathbb{D}}}

\def\sH{{\mathbb{H}}}
\def\sI{{\mathbb{I}}}

\def\sM{{\mathbb{M}}}

\def\sR{{\mathbb{R}}}

\def\sX{{\mathbb{X}}}

\def\sZ{{\mathbb{Z}}}

\newcommand{\E}{\mathbb{E}}

\newcommand{\KL}{\text{KL}}

\usepackage[capitalize]{cleveref}
\usepackage{algorithm}
\usepackage{algorithmic}
\usepackage{graphicx}
\usepackage{enumitem}
\usepackage{tabu}
\usepackage{nicefrac}
\usepackage{xcolor}
\usepackage{cancel}
\usepackage{amsthm}
\usepackage{subcaption}
\graphicspath{{./img/}}

\newcommand{\x}{\rvx}
\newcommand{\z}{\rvz}
\newcommand{\w}{\rvw}
\renewcommand{\L}{\gL}

\newtheorem{proposition}{Proposition}

\newcommand{\citep}{\cite}
\newcommand{\citet}{\cite}

\title{Bayesian Variational Autoencoders for \\Unsupervised Out-of-Distribution Detection}

\author{%
  Erik Daxberger\\
  University of Cambridge\\
  Max Planck Institute for\\
  Intelligent Systems, T\"{u}bingen\\
  \texttt{ead54@cam.ac.uk}\\
  \And
  Jos\'{e} Miguel Hern\'{a}ndez-Lobato\\
  University of Cambridge\\
  Alan Turing Institute\\
  Microsoft Research\\
  \texttt{jmh233@cam.ac.uk}\\
}

\begin{document}

\maketitle

\begin{abstract}
    Despite their successes, deep neural networks may make unreliable predictions when faced with test data drawn from a distribution different to that of the training data, constituting a major problem for AI safety.
    While this has recently motivated the development of methods to detect such out-of-distribution (OoD) inputs, a robust solution is still lacking.
    We propose a new probabilistic, unsupervised approach to this problem based on a Bayesian variational autoencoder model, which estimates a full posterior distribution over the decoder parameters using stochastic gradient Markov chain Monte Carlo, instead of fitting a point estimate.
    We describe how information-theoretic measures based on this posterior can then be used to detect OoD inputs both in input space and in the model's latent space.
    We empirically demonstrate the effectiveness of our proposed approach.
\end{abstract}
\section{Introduction}
\label{sec:intro}
\textbf{Outlier detection in input space.}
While deep neural networks (DNNs) have successfully tackled complex real-world problems in various domains including vision, speech and language \citep{lecun2015}, they still face significant limitations that make them unfit for safety-critical applications \citep{amodei2016}.
One well-known shortcoming of DNNs is when faced with test data points coming from a different distribution than the data the network saw during training, the DNN will not only output incorrect predictions, but it will do so with high confidence \cite{nguyen2015}.
The lack of robustness of DNNs to such \emph{out-of-distribution} (OoD) inputs (or \emph{outliers/anomalies}) was recently addressed by various methods to detect OoD inputs in the context of prediction tasks (typically classification)
\citep{hendrycks2016,liang2017,hendrycks2018}.
When we are only given input data, one simple and seemingly sensible approach to detect a potential OoD input $\x^*$ is to train a likelihood-based deep generative model (DGM; e.g. a VAE, auto-regressive DGM, or flow-based DGM) by (approximately) maximizing the probability $p(\sD|\theta)$ of the training data $\sD$ under the model parameters $\theta$, and to then estimate the density $p(\x^*|\theta)$ of $\x^*$ under the generative model $\theta$ \citep{bishop1994}.
If $p(\x^*|\theta)$ is large, then $\x^*$ is likely in-distribution, and OoD otherwise.
However, recent works have shown that this likelihood-based approach does not work in general, as DGMs sometimes assign \emph{higher} density to OoD data than to in-distribution data \citep{nalisnick2018}.
While some papers developed more effective scores that correct the likelihood \citep{choi2018,ren2019,nalisnick2019}, we argue and show that OoD detection methods
fundamentally based on the unreliable likelihood estimates by DGMs are not robust.

\textbf{Outlier detection in latent space.}
In a distinct line of research, recent works have tackled the challenge of optimizing a costly-to-evaluate black-box function $f: \sX \rightarrow \sR, f(\x) = y$ over a high-dimensional, richly structured input domain $\sX$ (e.g. graphs, images).
Given data $\sD = \{(\x_i, y_i)\}_{i=1}^N$, these methods jointly train a VAE on inputs $\x$ and a predictive model $g: \sZ \rightarrow \sR, g(\z) = y$ mapping from latent codes $\z$ to targets $y$, to then perform the optimization w.r.t.\ $y$ in the low-dimensional, continuous latent space $\sZ$ instead of in input space $\sX$ \citep{gomez2018}.
While these methods have achieved strong results in domains including automatic chemical design and automatic machine learning 
\citep{gomez2018,luo2018,lu2018,tripp2020sample}, their practical effectiveness is limited by their ability to handle the following trade-off: They need to find inputs $\x$ that both have a high target value $y$ and are sufficiently novel (i.e., not too close to training inputs $\sD$), and at the same time ensure that the optimization w.r.t.\ $y$ does not progress into regions of the latent space $\sZ$ too far away from the training data, which might yield latent points $\z$ that decode to semantically meaningless or syntactically invalid inputs $\x$
\citep{kusner2017}.
The required ability to quantify the \emph{novelty} of latents $\z$ (i.e., the semantic/syntactic distance to $\sD$) directly corresponds to the ability to effectively detect \emph{outliers} in latent space $\sZ$.

\textbf{Our approach.}
We propose a novel unsupervised, probabilistic method to \emph{simultaneously} tackle the challenge of detecting outliers $\x^*$ in input space $\sX$
as well as outliers $\z^*$ in latent space $\sZ$.
To this end, we take an information-theoretic perspective on OoD detection, and propose to use the (expected) \emph{informativeness} of an input $\x^*$ / latent $\z^*$ as a proxy for whether $\x^*$ / $\z^*$ is OoD or not.
To quantify this informativeness, we leverage probabilistic inference methods to maintain a posterior distribution over the parameters of a DGM, in particular of a variational autoencoder (VAE) \citep{kingma2013,rezende2014}.
This results in a \emph{Bayesian VAE} (BVAE) model, where instead of fitting a point estimate of the decoder parameters via maximum likelihood, we estimate their posterior using samples generated via stochastic gradient Markov chain Monte Carlo (MCMC).
The informativeness of an unobserved
$\x^*$ /
$\z^*$ is then quantified by measuring the (expected) change in the posterior over model parameters after having observed $\x^*$ / $\z^*$, revealing an intriguing connection to information-theoretic \emph{active learning} \citep{mackay1992}. 

\textbf{Our contributions.}
(a) We explain how DGMs can be made more robust by capturing epistemic uncertainty via a posterior distribution over their parameters, and describe how such \emph{Bayesian DGMs} can effectively detect outliers both in input space and in the model's latent space based on information-theoretic principles (\cref{sec:ood}).
(b) We propose a \emph{Bayesian VAE} model as a concrete instantiation of a Bayesian DGM (\cref{sec:bvae}).
(c) We empirically demonstrate that our approach significantly outperforms previous OoD detection methods across commonly-used benchmarks (\cref{sec:experiments}).

\section{Problem Statement and Background}
\label{sec:background}
\subsection{Out-of-Distribution (OoD) Detection}
For \emph{input space} OoD detection, we are given a large set $\sD = \{\x_i\}_{i=1}^N$ of high-dimensional training inputs $\x_i \in \sX$ (i.e., with $N > 25,000$ and $\text{dim}(\sX) > 500$) drawn i.i.d.\ from a distribution $p^*(\x)$, and a \emph{single} test input $\x^*$, and need to determine if $\x^*$ was drawn from $p^*$ or from some other distribution $\tilde{p} \neq p^*$.
\emph{Latent space} OoD detection is analogous, but with an often smaller set of typically lower-dimensional latent points $\z_i \in \sZ$ (i.e., with $\text{dim}(\sZ) < 100$).

\subsection{Variational Autoencoders}
Consider a latent variable model $p(\x, \z | \theta)$ with marginal log-likelihood (or \emph{evidence}) $\log p(\x|\theta) = \log \int p(\x, \z | \theta) d\z$, where $\x$ are observed variables, $\z$ are latent variables, and $\theta$ are model parameters.
We assume that $p(\x, \z | \theta) = p(\x|\z, \theta) p(\z)$ factorizes into a prior distribution $p(\z)$ over $\z$ and a likelihood $p(\x|\z, \theta)$ of $\x$ given $\z$ and $\theta$.
As we assume the $\z$ to be continuous, $p(\x|\theta)$ is intractable to compute.
We obtain a \emph{variational autoencoder} (VAE) \citep{kingma2013,rezende2014} if $\theta$ are the parameters of a DNN (the \emph{decoder}), and the resulting intractable posterior $p(\z|\x, \theta)$ over $\z$ is approximated using amortized variational inference (VI) via another DNN $q(\z|\x, \phi)$ with parameters $\phi$ (the \emph{encoder} or \emph{inference/recognition network}).
Given training data $\sD$, the parameters $\theta$ and $\phi$ of a VAE are learned by maximizing the evidence lower bound (ELBO) $\sum_{\x \in \sD} \mathcal{L}_{\theta, \phi}(\x)$, where
\begin{flalign}
\label{eq:elbo_vae}
\hspace{-3mm}
    \L_{\theta, \phi}(\x) \hspace{-1mm}= \hspace{-1mm}\E_{q(\z|\x,\phi)}[\log p(\x | \z, \theta)] \hspace{-0.5mm}- \hspace{-0.5mm}\KL[ {q(\z|\x, \phi)} \Vert p(\z) ]\hspace{-2mm}
\end{flalign}
for $\x \in \sD$, with $\L_{\theta,\phi}(\x) \leq \log p(\x|\theta)$.
As maximizing the ELBO approximately maximizes the evidence $\log p(\sD|\theta)$, this can be viewed as approximate maximum likelihood estimation (MLE).
In practice, $\L_{\theta, \phi}(\x)$ in \cref{eq:elbo_vae} is optimized by mini-batch stochastic gradient-based methods using low-variance, unbiased, stochastic Monte Carlo estimators of $\nabla \mathcal{L}_{\theta, \phi}$ obtained via the reparametrization trick.
Finally, one can use \emph{importance sampling} w.r.t.\ the variational posterior $q(\z|\x, \phi)$ to get an estimator $\hat{p}(\x|\theta, \phi)$ of the probability $p(\x|\theta)$ of an input $\x$ under the generative model, i.e.,
\begin{equation}
\label{eq:is_vae}
    p(\x|\theta) \simeq \hat{p}(\x | \theta, \phi)
    = \textstyle\frac{1}{K} \textstyle\sum_{k=1}^K \frac{p(\x|\z_k, \theta) p(\z_k)}{q(\z_k |\x, \phi)}\ ,
\end{equation}
where $\z_k \sim q(\z|\x, \phi)$ and where the estimator $\hat{p}(\x | \theta, \phi)$ is conditioned on \emph{both} $\theta$ and $\phi$ to make explicit the dependence on the parameters $\phi$ of the proposal distribution $q(\z|\x, \phi)$.

\section{Information-theoretic Out-of-Distribution Detection}
\label{sec:ood}
\subsection{Motivation and Intuition}
\label{subsec:motivation}

\textbf{Why do deep generative models \emph{fail} at OoD detection?}
Consider the
approach which first trains a density estimator parameterized by $\theta$, and then classifies an input $\x^*$ as OoD based on a threshold on the density of $\x^*$, i.e., if $p(\x^*|\theta) < \tau$ \cite{bishop1994}.
Recent advances in deep generative modeling (DGM) allow us to do density estimation even over high-dimensional, structured input domains (e.g. images, text), which in principle enables us to use this method in such complex settings.
However, the resulting OoD detection performance fundamentally relies on the quality of the likelihood estimates produced by these DGMs.
In particular, a sensible density estimator should assign high density to everything within the training data distribution, and low density to everything outside -- a property of crucial importance for effective OoD detection.
Unfortunately, \citet{nalisnick2018,choi2018} found that modern DGMs are often poorly calibrated, assigning \emph{higher} density to OoD data than to in-distribution data.%
\footnote{It is a common misconception that DGMs are "immune" to OoD miscalibration as they capture a density. While this might hold for simple models such as KDEs \cite{parzen1962} on low-dimensional data, it does \emph{not} generally hold for complex, DNN-based models on high-dimensional data \cite{nalisnick2018}.
In particular, while DGMs are trained to assign high probability to the training data, OoD data is not necessarily assigned low probability.}
This questions the use of DGMs for reliable density estimation and thus robust OoD detection.

\textbf{How are deep \emph{discriminative} models made OoD robust?}
DNNs are typically trained by maximizing the likelihood $p(\sD|\theta)$ of a set of training examples $\sD$ under model parameters $\theta$, yielding the maximum likelihood estimate (MLE) $\theta^*$ \cite{goodfellow2016}.
For \emph{discriminative}, predictive models $p(y|\x,\theta)$, it is well known that the point estimate $\theta^*$ does not capture \emph{model / epistemic uncertainty}, i.e., uncertainty about the choice of model parameters $\theta$ induced by the fact that many different models $\theta$ might have generated $\sD$.
As a result, discriminative models $p(y|\x,\theta^*)$ trained via MLE tend to be overconfident in their predictions, especially on OoD data \cite{nguyen2015,guo2017}.
A principled, established way to capture model uncertainty in DNNs is to be Bayesian and infer a full \emph{distribution} $p(\theta|\sD)$ over parameters $\theta$, yielding the predictive distribution $p(y|\x,\sD) = \int p(y|\x,\theta) p(\theta|\sD) d\theta$ \cite{gal2016uncertainty}.
Bayesian DNNs have much better OoD robustness than deterministic ones (e.g., producing low uncertainty for in-distribution and high uncertainty for OoD data), and OoD calibration has become a major benchmark for Bayesian DNNs \cite{gal2016,ovadia2019,osawa2019,maddox2019}.
This suggests that capturing model uncertainty via Bayesian inference is a promising way to achieve robust, principled OoD detection.

\textbf{Why should we use \emph{Bayesian DGMs} for OoD detection?}
Just like deep discriminative models, DGMs are typically trained by maximizing the probability $p(\sD|\theta)$ that $\sD$ was generated by the density model, yielding the MLE $\theta^*$.
As a result, it is not surprising that the shortcomings of MLE-trained \emph{discriminative} models also translate to MLE-trained \emph{generative} models, such as the miscalibration and unreliability of their likelihood estimates $p(\x^*|\theta^*)$ for OoD inputs $\x^*$.
This is because there will always be many different plausible generative / density models $\theta$ of the training data $\sD$, which are not captured by the point estimate $\theta^*$.
If we do not trust our predictive models $p(y|\x,\theta^*)$ on OoD data, why should we trust our generative models $p(\x|\theta^*)$, given that both are based on the same, unreliable DNNs?
In analogy to the discriminative setting, we argue that OoD robustness can be achieved by capturing the epistemic uncertainty in the DGM parameters $\theta$.
This motivates the use of \emph{Bayesian} DGMs, which estimate a full \emph{distribution} $p(\theta|\sD)$ over parameters $\theta$ and thus capture many different density estimators to explain the data, yielding the \emph{expected/average} likelihood
\begin{equation}
\label{eq:expected_lik}
    p(\x|\sD) = \textstyle\int p(\x|\theta) p(\theta|\sD) d\theta = \E_{p(\theta|\sD)}[p(\x|\theta)] \ .
\end{equation}

\textbf{\emph{How} can we use Bayesian DGMs for OoD detection?}
Assume that given data $\sD$, we have inferred a distribution $p(\theta|\sD)$ over the parameters $\theta$ of a DGM.
In particular, we consider the case where $p(\theta|\sD)$ is represented by a set $\{ \theta_m \}_{m=1}^M$ of $M$ \emph{samples} $\theta_m \sim p(\theta|\sD)$, which can be viewed as an \emph{ensemble} of $M$ DGMs.
Our goal now is to decide if a given, new input $\x^*$ is in-distribution or OoD.
To this end, we \emph{refrain} from classifying $\x^*$ as OoD based on a threshold on the (miscalibrated) likelihoods that one or more of the models $\{ \theta_m \}_{m=1}^M$ assign to $\x^*$ \cite{bishop1994}.%
\footnote{Note that $p(\x^*|\sD) \simeq \frac{1}{M} \sum_{m=1}^M p(\x^*|\theta_m)$ in \cref{eq:expected_lik} remains unreliable if $\x^*$ is OoD. E.g., \citet{nalisnick2018} show that averaging likelihoods across an ensemble of DGMs \emph{does not help}.}
\textbf{Instead, we propose to use a threshold on a measure $D[\cdot]$ of the \emph{variation} or \emph{disagreement} in the likelihoods $\{p(\x^*|\theta_m)\}_{m=1}^M$ of the different models $\{ \theta_m \}_{m=1}^M$, i.e., to classify an input $\x^*$ as OoD if $D[\{p(\x^*|\theta_m)\}_{m=1}^M] < \tau$}.
In particular, if the models $\{ \theta_m \}_{m=1}^M$ \emph{agree} as to how probable $\x^*$ is, then $\x^*$ likely is an \emph{in-distribution} input.
Conversely, if the models $\{ \theta_m \}_{m=1}^M$ \emph{disagree} as to how probable $\x^*$ is, then $\x^*$ likely is an \emph{OoD} input.%
\footnote{If the $\{ \theta_m \}_{m=1}^M$ were perfect density estimators, they would all \emph{agree} that an OoD input $\x^*$ is \emph{unlikely}. But, model uncertainty makes the DGMs \emph{disagree} on $\{p(\x^*|\theta_m)\}_{m=1}^M$ if $\x^*$ is OoD.}
This intuitive decision rule is a direct consequence of the property that the epistemic uncertainty of a parametric model $\theta$ (which is exactly what the variation/disagreement across models $\{ \theta_m \}_{m=1}^M$ captures) is naturally low for in-distribution and high for OoD inputs -- the very same property that makes Bayesian discriminative DNNs robust to OoD inputs. 

\subsection{Quantifying Disagreement between Models}
\label{subsec:ood_input}
We propose the following score $D_\Theta[\x^*]$ to quantify the disagreement or variation in the likelihoods $\{p(\x^*|\theta_m)\}_{m=1}^M$ of a set $\{ \theta_m \}_{m=1}^M$ of model parameter samples $\theta_m \sim p(\theta|\sD)$:
\begin{flalign}
\label{eq:D_ess}
    D_\Theta[\x^*] = \textstyle\frac{1}{\sum_{\theta \in \Theta} w_\theta^2}, \quad \text{with} \quad w_\theta = \textstyle\frac{p(\x^*|\theta)}{\sum_{\theta \in \Theta} p(\x^*|\theta)}\ .
\end{flalign}
I.e., the likelihoods $\{p(\x^*|\theta_m)\}_{m=1}^M$ are first \emph{normalized} to yield $\{w_{\theta_m}\}_{m=1}^M$ (see \cref{eq:D_ess}), such that $w_\theta \in [0,1]$ and $\sum_{\theta \in \Theta} w_\theta = 1$.
The normalized likelihoods $\{w_{\theta_m}\}_{m=1}^M$ effectively define a categorical distribution over models $\{ \theta_m \}_{m=1}^M$, where each value $w_\theta$ can be interpreted as the probability that $\x^*$ was generated from the model $\theta$, thus measuring how well $\x^*$ is \emph{explained} by the model $\theta$, relative to the other models.
To obtain the score $D_\Theta[\x^*]$ in \cref{eq:D_ess}, we then square the normalized likelihoods, sum them up, and take the reciprocal.
Note that $D_\Theta[\x^*] \in [1,M], \forall \x^*$.%
For \emph{latent points} $\z^* \in \sZ$, we take into account all possible inputs $\x^* \in \sX$ corresponding to $\z^*$, yielding the \emph{expected} disagreement $D_\Theta[\z^*] =  \E_{p(\x|\z^*)}\left[ D_\Theta[\x] \right] \simeq \textstyle\frac{1}{N} \textstyle\sum_{n=1}^N D_\Theta[\x_n]$, with inputs $\x_n$ sampled from the conditional distribution $p(\x|\z^*)$, and $D_\Theta[\x]$ defined as in \cref{eq:D_ess}.

As $D_\Theta[\cdot]$ measures the degree of disagreement
between the models $\{ \theta_m \}_{m=1}^M$ as to how probable $\x^*$/$\z^*$ is, it can be used to classify $\x^*$/$\z^*$ as follows:
If $D_\Theta[\cdot]$ is large, then $[w_\theta]_{\theta \in \Theta}$ is close to the (discrete) uniform distribution $[\frac{1}{M}]_{\theta \in \Theta}$ (for which $D_\Theta[\cdot] = M$), meaning that all models $\theta \in \Theta$ explain $\x^*$/$\z^*$ equally well and are in agreement as to how probable $\x^*$/$\z^*$ is. Thus, $\x^*$/$\z^*$ likely is \emph{in-distribution}.
Conversely, if $D_\Theta[\cdot]$ is small, then $[w_\theta]_{\theta \in \Theta}$ contains a few large weights (i.e., corresponding to models that by chance happen to explain $\x^*$/$\z^*$ well), with all other weights being very small, where in the extreme case, $[w_\theta]_{\theta \in \Theta} = [0, \ldots, 0, 1, 0, \ldots, 0]$ (for which $D_\Theta[\cdot] = 1$). This means that the models do not agree as to how probable $\x^*$/$\z^*$ is, so that $\x^*$/$\z^*$ likely is \emph{OoD}.

As we argue in more detail in \cref{sec:info_perspective}, there is a principled justification for the disagreement score $D_\Theta[\x^*]$ in \cref{eq:D_ess}, which induces an \emph{information-theoretic perspective} on OoD detection.
In particular, $D_\Theta[\x^*]$ can be viewed as quantifying the \emph{informativeness} of $\x^*$ for updating the DGM parameters $\theta$ to the ones capturing the true density.
The OoD detection mechanism described above can thus be intuitively summarised as follows:
\emph{In-distribution} inputs $\x^*$ are similar to the data points already in $\sD$ and thus \emph{uninformative} about the model parameters $\theta$, inducing \emph{small} change in the posterior distribution $p(\theta|\sD)$, resulting in a \emph{large} score $D_\Theta[\x^*]$.
Conversely, \emph{OoD} inputs $\x^*$ are very different from the previous observations in $\sD$ and thus \emph{informative} about the model parameters $\theta$, inducing \emph{large} change in the posterior $p(\theta|\sD)$, resulting in a \emph{small} score $D_\Theta[\x^*]$.
This perspective on OoD detection reveals a close relationship to \emph{information-theoretic active learning} \cite{mackay1992,houlsby2011}.
There, the same notion of \emph{informativeness} (or, equivalently, \emph{disagreement}) is used to quantify the \emph{novelty} of an input $\x^*$ to be added to the data $\sD$, aiming to maximally improve the estimate of the model parameters $\theta$ by maximally \emph{reducing the entropy / epistemic uncertainty} in the posterior $p(\theta|\sD)$.

\section{The Bayesian Variational Autoencoder (BVAE)}
\label{sec:bvae}
As an example of a Bayesian DGM, we propose a \emph{Bayesian VAE} (BVAE), where instead of fitting the model parameters $\theta$ via (approximate) MLE, $\theta_{\text{MLE}} = \operatorname{arg\,max}_\theta \L_{\theta,\phi}(\sD)$, to get the likelihood $p(\x|\z, \theta_{\text{MLE}})$, we place a prior $p(\theta)$ over $\theta$ and estimate its posterior $p(\theta|\sD) \propto p(\sD|\theta) p(\theta)$, yielding the likelihood $p(\x|\z, \sD) = \int p(\x | \z, \theta) p(\theta| \sD) d\theta$.
The \emph{marginal likelihood}
$p(\x | \sD) = \textstyle\int\textstyle\int p(\x | \z, \theta) p(\z) d\z p(\theta| \sD) d\theta$
thus integrates out \emph{both} the latent variables $\z$ and model parameters $\theta$ (cf.\ \cref{eq:expected_lik}).
The resulting generative process draws a $\z \sim p(\z)$ from its prior and a $\theta \sim p(\theta|\sD)$ from its posterior, and then generates $\x \sim p(\x | \z, \theta)$.%
Training a BVAE thus requires Bayesian inference of \emph{both} the posterior $p(\z|\x, \sD)$ over $\z$ and the posterior $p(\theta | \sD)$ over $\theta$, which is both intractable and thus requires approximation.
We propose two variants for inferring
those posteriors
in a BVAE.
\begin{figure}[h]
	\begin{subfigure}[c]{0.5\textwidth}
	    \centering
    	\includegraphics[trim={4mm 2mm 5mm 3mm},clip,width=0.45\columnwidth]{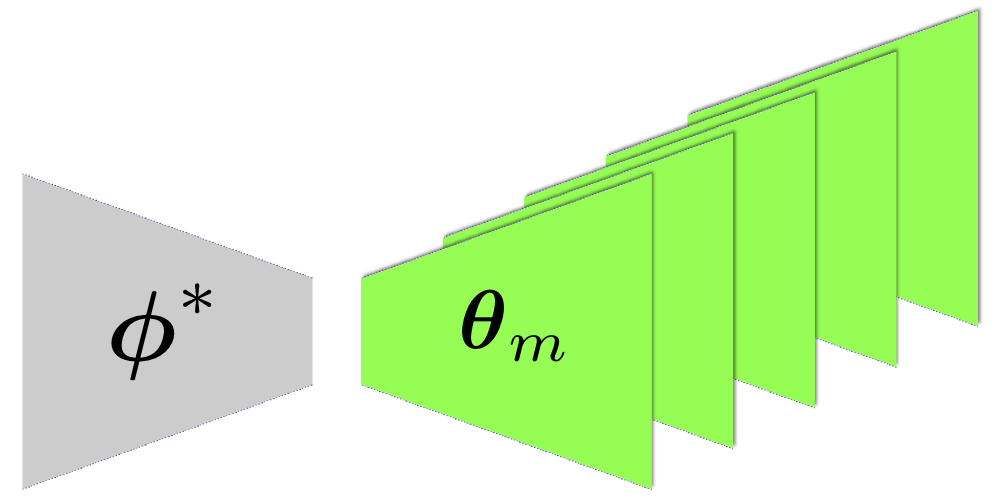}
    	\includegraphics[trim={4mm 5mm 5mm 5mm},width=0.45\columnwidth]{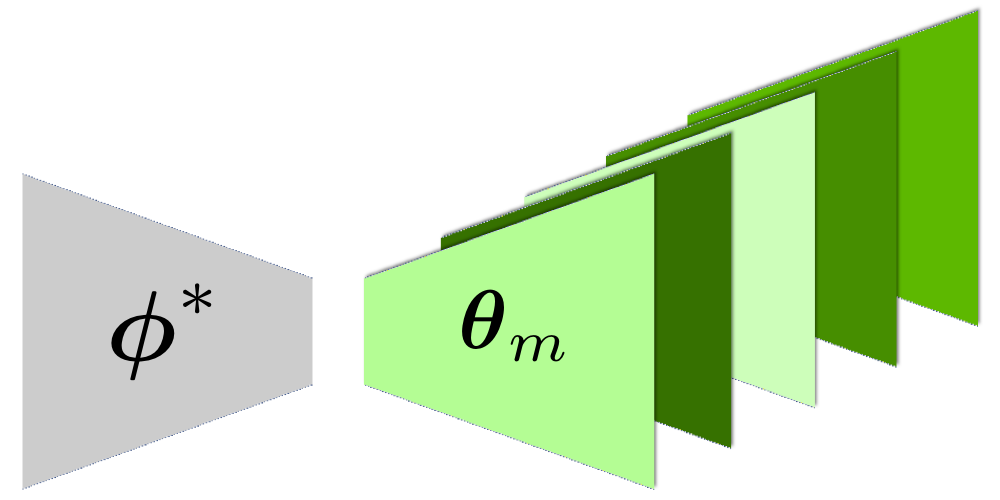}
    	\subcaption{Variant 1 with a shared encoder $\phi^*$ and $M$ decoders}
	\end{subfigure}
	\begin{subfigure}[c]{0.5\textwidth}
	    \centering
    	\includegraphics[trim={5mm 2mm 5mm 3mm},width=0.45\columnwidth]{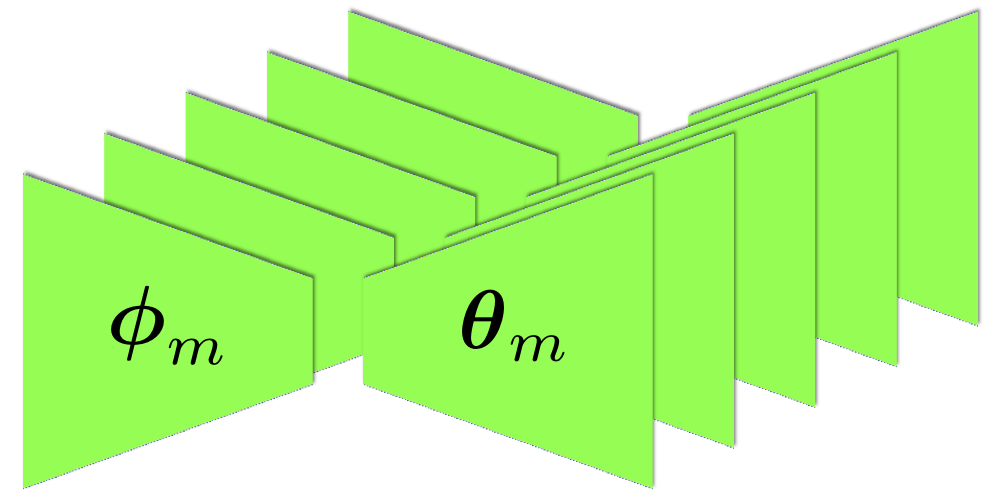}
    	\includegraphics[trim={5mm 5mm 5mm 5mm},width=0.45\columnwidth]{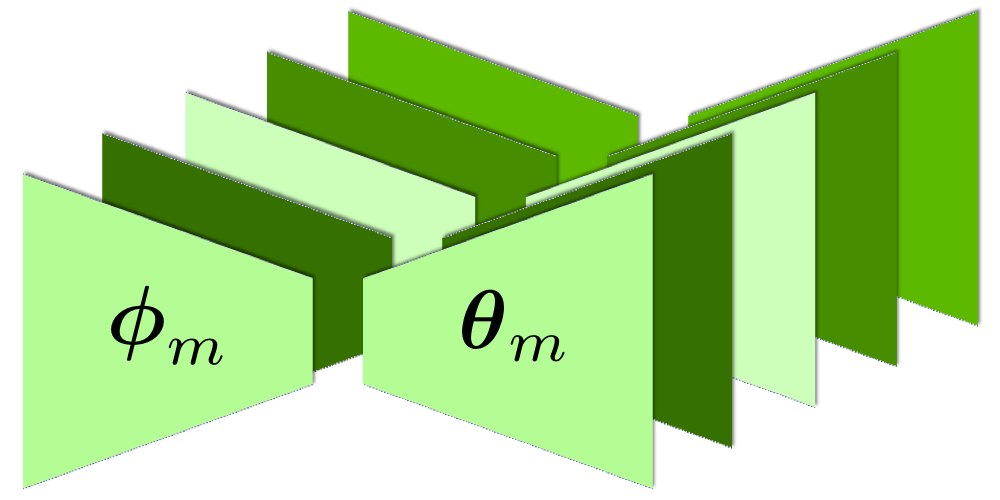}
    	\subcaption{Variant 2 as an ensemble of $M$ VAEs $(\phi_m, \theta_m)$}
	\end{subfigure}
	\caption{Illustrations of the (a) first and (b) second BVAE variant (with $M=5$), with (left) \emph{agreement} and (right) \emph{disagreement} in their likelihoods $p(\x^*|\theta_m)$ (as encoded by color intensity).}
	\label{fig:bvae}
\end{figure}
\subsection{Variant 1: BVAE with a Single Fixed Encoder}
\label{subsec:bvae_variant_1}
\textbf{a) Learning the encoder parameters $\phi$.}
As in a regular VAE, we approximate the posterior $p(\z | \x, \sD)$ using amortized VI via an inference network $q(\z | \x, \phi)$ whose parameters $\phi$ are fit by maximizing the ELBO $\L_{\theta,\phi}(\x)$ in \cref{eq:elbo_vae}, $\phi^* = \operatorname{arg\ max} \L_{\theta, \phi}(\x)$, yielding a single fixed encoder.

\textbf{b) Learning the decoder parameters $\theta$.}
\label{subsubsec:inference_theta}
To generate posterior samples $\theta \sim p(\theta|\sD)$ of decoder parameters,
we propose to use SGHMC (see \cref{sec:sgmcmc}).
However, the gradient of the energy function $\nabla_\theta U(\theta, \sM)$ in \cref{eq:grad_U_M} used for simulating the Hamiltonian dynamics requires evaluating the log-likelihood $\log p(\x|\theta)$, which is intractable in a (B)VAE.
To alleviate this, we approximate the log-likelihood in $\nabla_\theta U(\theta, \sM)$ by the ordinary VAE ELBO $\L_{\theta,\phi}(\x)$ in \cref{eq:elbo_vae}.
Given a set $\Theta = \{\theta_m\}_{m=1}^M$ of posterior samples $\theta_m \sim p(\theta | \sD)$, we can more intuitively think of having a finite mixture/ensemble of decoders/generative models $p(\x|\z, \sD) = \E_{p(\theta|\sD)}[p(\x | \z, \theta)] \simeq \frac{1}{M} \sum_{\theta \in \Theta} p(\x | \z, \theta)$.

\textbf{c) Likelihood estimation.}
This BVAE variant is effectively trained like a normal VAE, but using a sampler instead of an optimizer for $\theta$ (pseudocode is found in the appendix).
We obtain an ensemble of $M$ VAEs $(\phi^*, \theta_m)$ with a single \emph{shared} encoder $\phi^*$ and $M$ separate decoder samples $\theta_m$; see \cref{fig:bvae} (left) for a cartoon illustration.
For the $m$-th VAE, the likelihood $p(\x|\theta_m) \simeq \hat{p}(\x | \theta_m, \phi^*)$ can then be estimated via importance sampling w.r.t.\ $q(\z|\x,\phi^*)$, as in \cref{eq:is_vae}.
\subsection{Variant 2: BVAE with a \emph{Distribution} over Encoders}
\label{subsec:bvae_variant_2}
\textbf{a) Learning the encoder parameters $\phi$.}
\label{subsubsec:inference_phi}
Recall that amortized VI aims to \emph{learn} to do posterior inference, by optimizing the parameters $\phi^* = \operatorname{arg\,max}_\phi \L(\sD)_{\theta,\phi}$ (see \cref{eq:elbo_vae}) of an inference network
$i_\phi(\x) = \psi$
mapping inputs $\x$ to parameters $\psi$ of the variational posterior $q_\psi(\z) = q(\z|\x,\phi)$ over $\z$.
However, one major shortcoming of fitting a single encoder parameter setting $\phi^*$ is that $q(\z|\x,\phi^*)$ will \emph{not} generalize to OoD inputs, but will instead produce confidently wrong posterior inferences \cite{cremer2018,mattei2018} (cf.\ \cref{subsec:motivation}).
To alleviate this, we instead capture multiple encoders by inferring a \emph{distribution} over the variational parameters $\phi$.
While this might appear odd conceptually, it allows us to quantify our \emph{epistemic uncertainty in the amortized inference} of $\z$.
It might also be interpreted as regularizing the encoder \cite{shu2018}, or as increasing its flexibility \citep{yin2018semi}.
We thus also place a prior $p(\phi)$ over $\phi$ and infer the posterior $p(\phi|\sD)$, yielding the amortized posterior
$q(\z | \x, \sD) = \int q(\z | \x, \phi) p(\phi|\sD) d\phi$.
We also use SGHMC to sample $\phi_m \sim p(\phi | \sD)$, again using the ELBO $\L_{\theta,\phi}(\x)$ in \cref{eq:elbo_vae} to compute $\nabla_\phi U(\phi, \sM)$ (see \cref{eq:grad_U_M} in \cref{sec:sgmcmc}).
Given a set $\Phi = \{\phi_m\}_{m=1}^M$ of posterior samples $\phi_m \sim p(\phi|\sD)$, we can again more intuitively think of having as a finite mixture/ensemble of encoders/inference networks $q(\z | \x, \sD) = \E_{p(\phi|\sD)}[q(\z | \x, \phi)] \simeq \frac{1}{M} \sum_{\phi \in \Phi} q(\z | \x, \phi)$.

\textbf{b) Learning the decoder parameters $\theta$.}
We sample $\theta \sim p(\theta|\sD)$ as in \cref{subsubsec:inference_theta}.
The only difference is that we now have the encoder mixture $q(\z|\x,\sD)$ instead of the single encoder $q(\z|\x,\phi)$, technically yielding the ELBO
$\L_{\theta}(\x) = \E_{q(\z | \x, \sD)} [\log p(\x | \z, \theta)] - \KL[q(\z|\x,\sD) \Vert p(\z)]$
which depends on $\theta$ only, as %
$\phi$ is averaged
over $p(\phi|\sD)$.
However, in practice, we for simplicity only use the most recent sample $\phi_m \sim p(\phi|\sD)$ to estimate $q(\z | \x, \sD) \simeq q(\z | \x, \phi_m)$, such that $\L_{\theta}(\x)$ effectively reduces to the normal VAE ELBO in \cref{eq:elbo_vae} with fixed encoder $\phi_m$.

\textbf{c) Likelihood estimation.}
This BVAE variant is effectively trained like a normal VAE, but using a sampler instead of an optimizer for \emph{both} $\phi$ and $\theta$ (pseudocode is found in the appendix).
We obtain an ensemble of $M$ VAEs $(\phi_m, \theta_m)$ with $M$ pairs of coupled encoder-decoder samples; see \cref{fig:bvae} (right) for a cartoon illustration.
For the $m$-th VAE, the likelihood $p(\x | \theta_m) \simeq \hat{p}(\x | \theta_m, \phi_m)$ can then be estimated via importance sampling w.r.t.\ $q(\z|\x,\phi_m)$, as in \cref{eq:is_vae}.

\section{Experiments}
\label{sec:experiments}
\subsection{Out-of-Distribution Detection in Input Space}
\label{subsec:exp_input}
\textbf{BVAE details.}
We assess both proposed BVAE variants: \texttt{BVAE$_1$} samples $\theta$ and optimizes $\phi$ (see \cref{subsec:bvae_variant_1}), while
\texttt{BVAE$_2$} samples \emph{both} $\theta$ and $\phi$ (see \cref{subsec:bvae_variant_2}).
Our \texttt{PyTorch} implementation uses Adam \cite{kingma2014adam} with learning rate $10^{-3}$ for optimization, and scale-adapted SGHMC with step size $10^{-3}$ and momentum decay $0.05$ \cite{springenberg2016} for sampling%
\footnote{We use the implementation of SGHMC
at \url{https://github.com/automl/pybnn}.}. %
Following \citet{chen2014,springenberg2016}, we place Gaussian priors over $\theta$ and $\phi$, i.e., $p(\theta) = \mathcal{N}(0, \lambda_\theta^{-1})$ and $p(\phi) = \mathcal{N}(0, \lambda_\phi^{-1})$, and Gamma hyperpriors over the precisions $\lambda_\theta$ and $\lambda_\phi$, i.e., $p(\lambda_\theta) = \Gamma(\alpha_\theta, \beta_\theta)$ and $p(\lambda_\phi) = \Gamma(\alpha_\phi, \beta_\phi)$, with $\alpha_\theta = \beta_\theta = \alpha_\phi = \beta_\phi = 1$, and resample $\lambda_\theta$ and $\lambda_\phi$ after every training epoch (i.e., a full pass over $\sD$).
We discard samples within a burn-in phase of $B = 1$ epoch and store a sample
after every $D = 1$ epoch, which we found to be robust and effective choices.

\textbf{Experimental setup.}
Following previous works, we use three benchmarks:
(a) FashionMNIST (in-distribution) vs.\ MNIST (OoD) \citep{hendrycks2018,nalisnick2018,zenati2018,akcay2018,ren2019},
(b) SVHN (in-distribution) vs.\ CIFAR10 (OoD)%
\footnote{Unlike \citet{nalisnick2018}, we found the likelihood calibration to be poor on this benchmark (\cref{fig:curves}, bottom middle, shows the overlap in likelihoods) and decent on the opposite benchmark.}
\citep{hendrycks2018,nalisnick2018,choi2018}, and
(c) eight classes of FashionMNIST (in-distribution) vs.\ the remaining two classes (OoD), using five different splits $\{(0,1), (2,3), (4,5), (6,7), (8,9)\}$ of held-out classes \citep{ahmed2019}.
We compare against the log-likelihood (\texttt{LL}) as well as all three state-of-the-art methods for unsupervised OoD detection described in \cref{sec:related_work}:
(1) The generative ensemble based method by \citet{choi2018} composed of five independently trained models (\texttt{WAIC}),
(2) the likelihood ratio method by \citet{ren2019} (\texttt{LLR}), using Bernoulli rates $\mu = 0.2$ for the FashionMNIST vs.\ MNIST benchmark \citep{ren2019}, and $\mu = 0.15$ for the other benchmarks, and
(3) the test for typicality by \citet{nalisnick2019} (\texttt{TT}).
All methods use VAEs for estimating log-likelihoods.%
\footnote{Note that like most previous OoD detection approaches, our method is agnostic to the specific deep generative model architecture used, and can thus be straightforwardly combined with any state-of-the-art architecture for maximal effectiveness in practice. For comparability, and to isolate the benefit or our method, we follow the same experimental protocol and use the same convolutional VAE architecture as in previous works \citet{nalisnick2018,choi2018,ren2019}.}
For evaluation, we randomly select $5000$ in-distribution and OoD inputs from held-out test sets and compute the following, threshold independent metrics \citep{hendrycks2016,liang2017,hendrycks2018,alemi2018,ren2019}: (i) The area under the ROC curve (AUROC$\uparrow$), (ii) the area under the precision-recall curve (AUPRC$\uparrow$), and (iii) the false-positive rate at $80\%$ true-positive rate (FPR80$\downarrow$).

\textbf{Results.}
\cref{tab:results} shows that both BVAE variants \emph{significantly outperform} the other methods on the considered benchmarks.
\cref{fig:curves} shows the ROC curves used to compute the AUROC metric in \cref{tab:results}, for the FashionMNIST vs.\ MNIST (top left) and SVHN vs.\ CIFAR10 (bottom left) benchmarks; ROC curves for the FashionMNIST (held-out) benchmark as well as precision-recall curves for all benchmarks are found in \cref{sec:additional_exp}.
\texttt{BVAE$_2$} outperforms \texttt{BVAE$_1$} on FashionMNIST vs.\ MNIST and SVHN vs.\ CIFAR10, where in-distribution and OoD data is very distinct, but not on FashionMNIST (held-out), where the datasets are much more similar.
This suggests that capturing a distribution over encoders $\phi$ is particularly beneficial when train and test data live on different manifolds (as overfitting $\phi$ is more critical), while the fixed encoder $\phi^*$ generalizes better when train and test manifolds are similar, which is as expected intuitively.
Finally, \cref{fig:curves} shows histograms of the log-likelihoods (top middle) and of the \texttt{BVAE$_2$} scores (top right) on FashionMNIST in-distribution (blue) vs.\ MNIST OoD (orange).
While the log-likelihoods strongly overlap, our proposed score more clearly separates in-distribution data (closer to the r.h.s.) from OoD data (closer to the l.h.s.).
The corresponding histograms for the SVHN vs.\ CIFAR10 task, showing similar behaviour, are shown in \cref{sec:additional_exp}.
Finally, note that to further boost performance (of all methods, not just ours), one can simply use a deep generative model architecture that is more sophisticated than the convolutional VAE that is part of the established experimental protocol which we follow for comparability \citet{nalisnick2018,choi2018,ren2019}.
\begin{table*}[ht]
    \footnotesize
    \caption{AUROC$\uparrow$, AUPRC$\uparrow$, and FPR80$\downarrow$ scores (where higher $\uparrow$ or lower $\downarrow$ is better) of our methods (top two rows) and the baselines (bottom four rows) . For the experiment on FashionMNIST with with held-out classes, we report the mean scores over all five class splits.}
    \label{tab:results}
    \centering
    \begin{tabu}{lcccccccccccc}
        \toprule
        & \multicolumn{3}{c}{\textbf{FashionMNIST vs MNIST}} & \multicolumn{3}{c}{\textbf{SVHN vs CIFAR10}} & \multicolumn{3}{c}{\textbf{FashionMNIST (held-out)}}\vspace{0.5mm}\\
        \cline{2-10}\vspace{-3mm}\\
        & AUROC & AUPRC & FPR80 & AUROC & AUPRC & FPR80 & AUROC & AUPRC & FPR80\\
        \midrule
        \texttt{BVAE$_1$} & 0.904 & 0.891 & 0.117 & 0.807 & 0.793 & 0.331 & \textbf{0.693} & \textbf{0.680} & \textbf{0.540}\\
        \texttt{BVAE$_2$} & \textbf{0.921} & \textbf{0.907} & \textbf{0.082} & \textbf{0.814} & \textbf{0.799} & \textbf{0.310} & 0.683 & 0.668 & 0.558\\
        \midrule
        \texttt{LL} & 0.557 & 0.564 & 0.703 & 0.574 & 0.575 & 0.634 & 0.565 & 0.577 & 0.683\\
        \texttt{LLR} & 0.617 & 0.613 & 0.638 & 0.570 & 0.570 & 0.638 & 0.560 & 0.569 & 0.698\\
        \texttt{TT} & 0.482 & 0.502 & 0.833 & 0.395 & 0.428 & 0.859 & 0.482 & 0.496 & 0.806\\
        \texttt{WAIC} & 0.541 & 0.548 & 0.798 & 0.293 & 0.380 & 0.912 & 0.446 & 0.464 & 0.827\\
        \bottomrule
    \end{tabu}
\end{table*}
\begin{figure*}%
	\centering
    \raisebox{-0.5\height}{\includegraphics[trim={2mm 8mm 0 2.5mm},clip,width=0.35\textwidth]{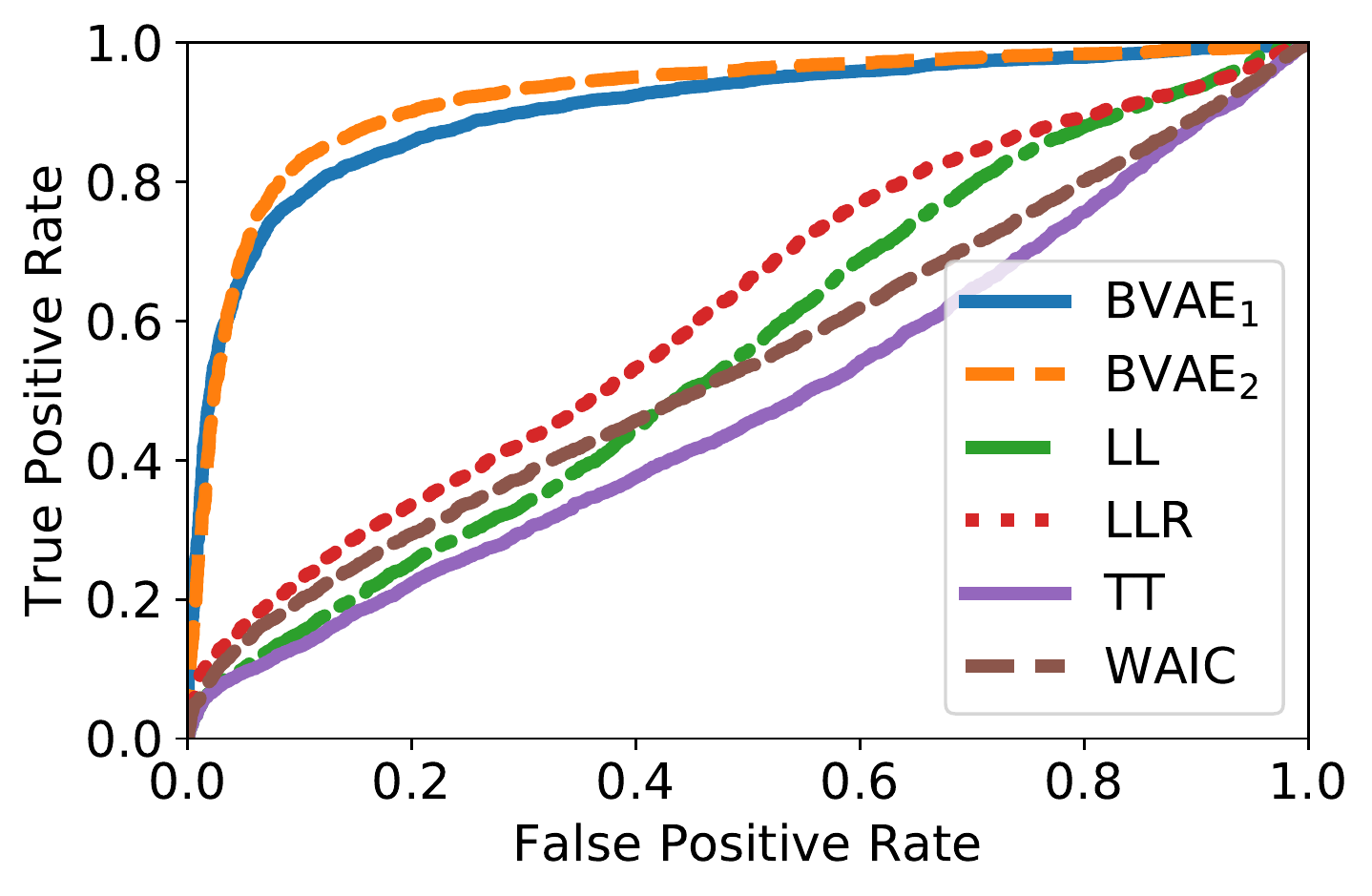}}
	\raisebox{-0.48\height}{\includegraphics[trim={8mm 0 0 0},clip,width=0.32\textwidth]{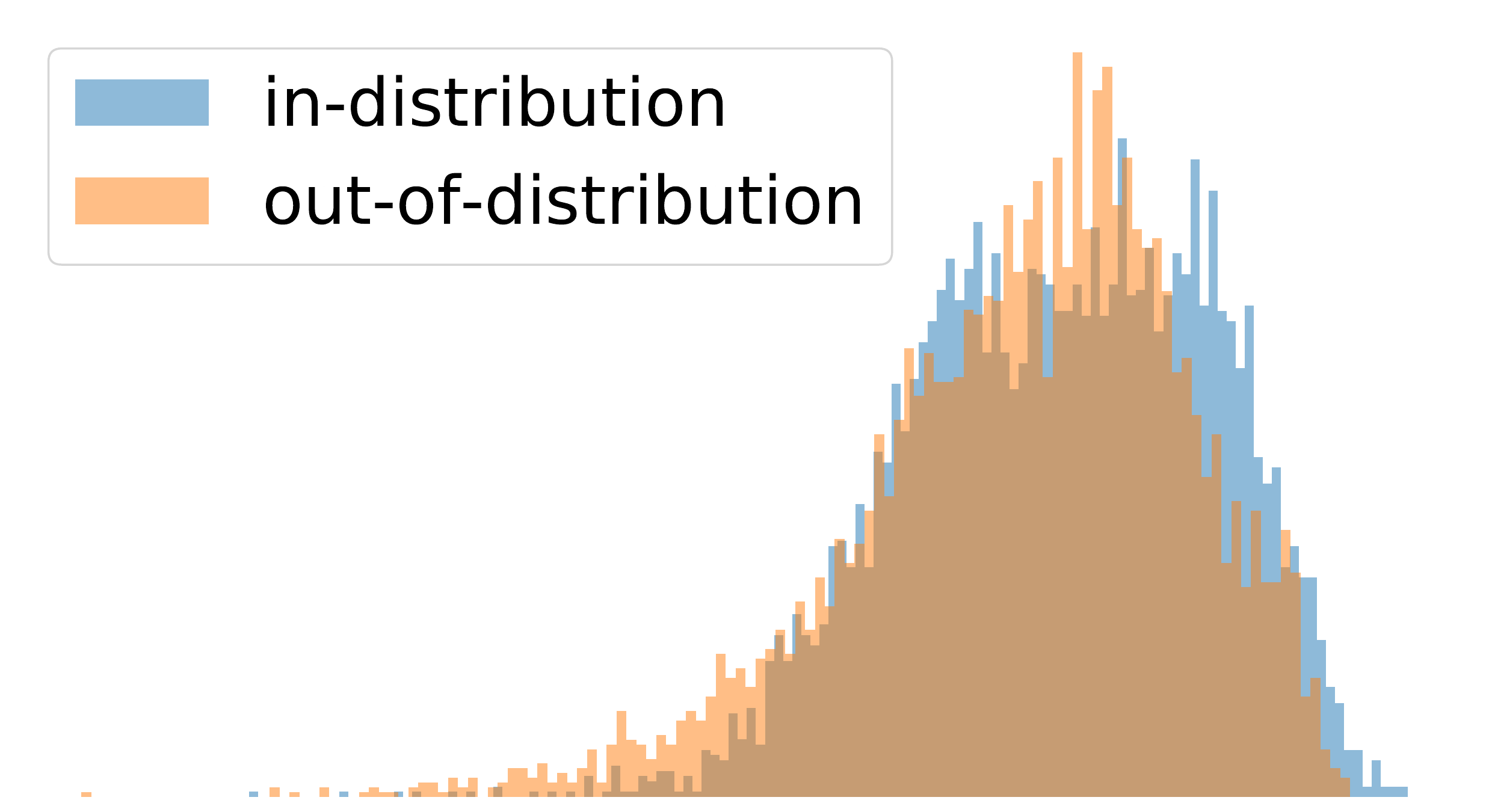}}
	\raisebox{-0.5\height}{\includegraphics[trim={12mm 0 7.5mm 0},clip,width=0.29\textwidth]{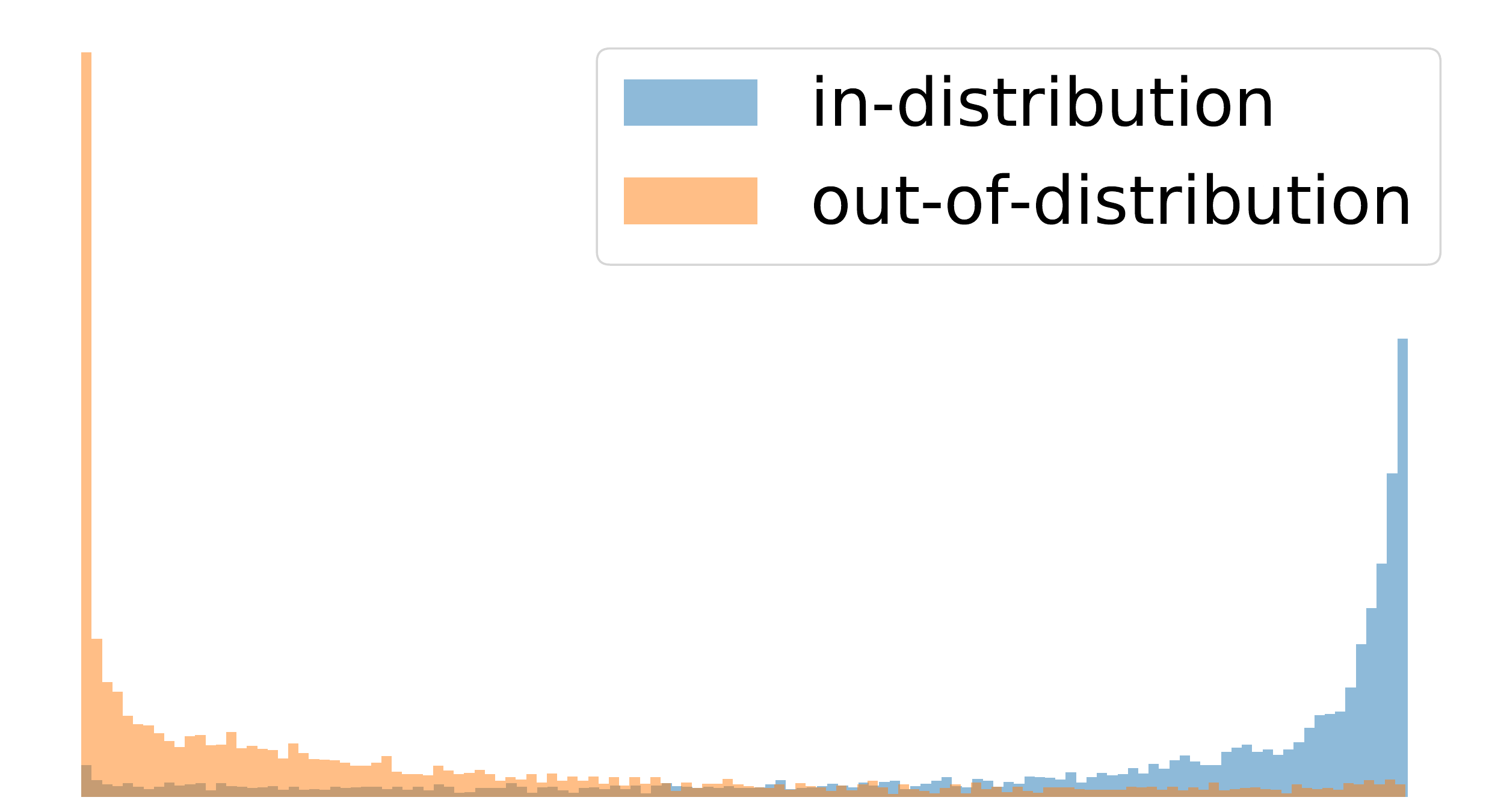}}
    \hspace{-40mm}
    \raisebox{-0.5\height}{\includegraphics[trim={2mm 3.5mm 0 2.5mm},clip,width=0.35\textwidth]{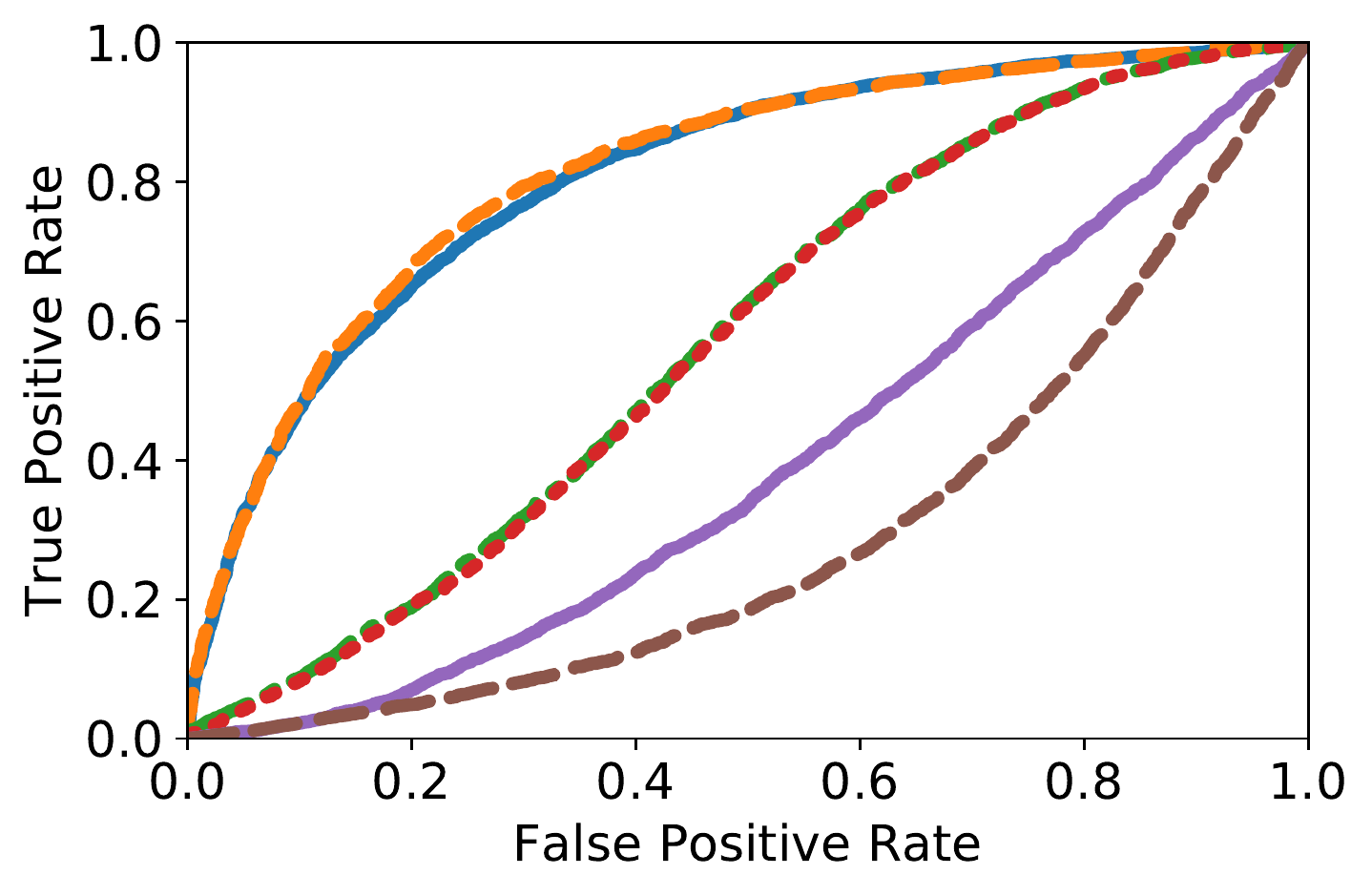}}
    \hspace{20mm}
	\raisebox{-0.46\height}{\includegraphics[trim={318mm 0 1mm 5.5mm},clip,width=0.30\textwidth]{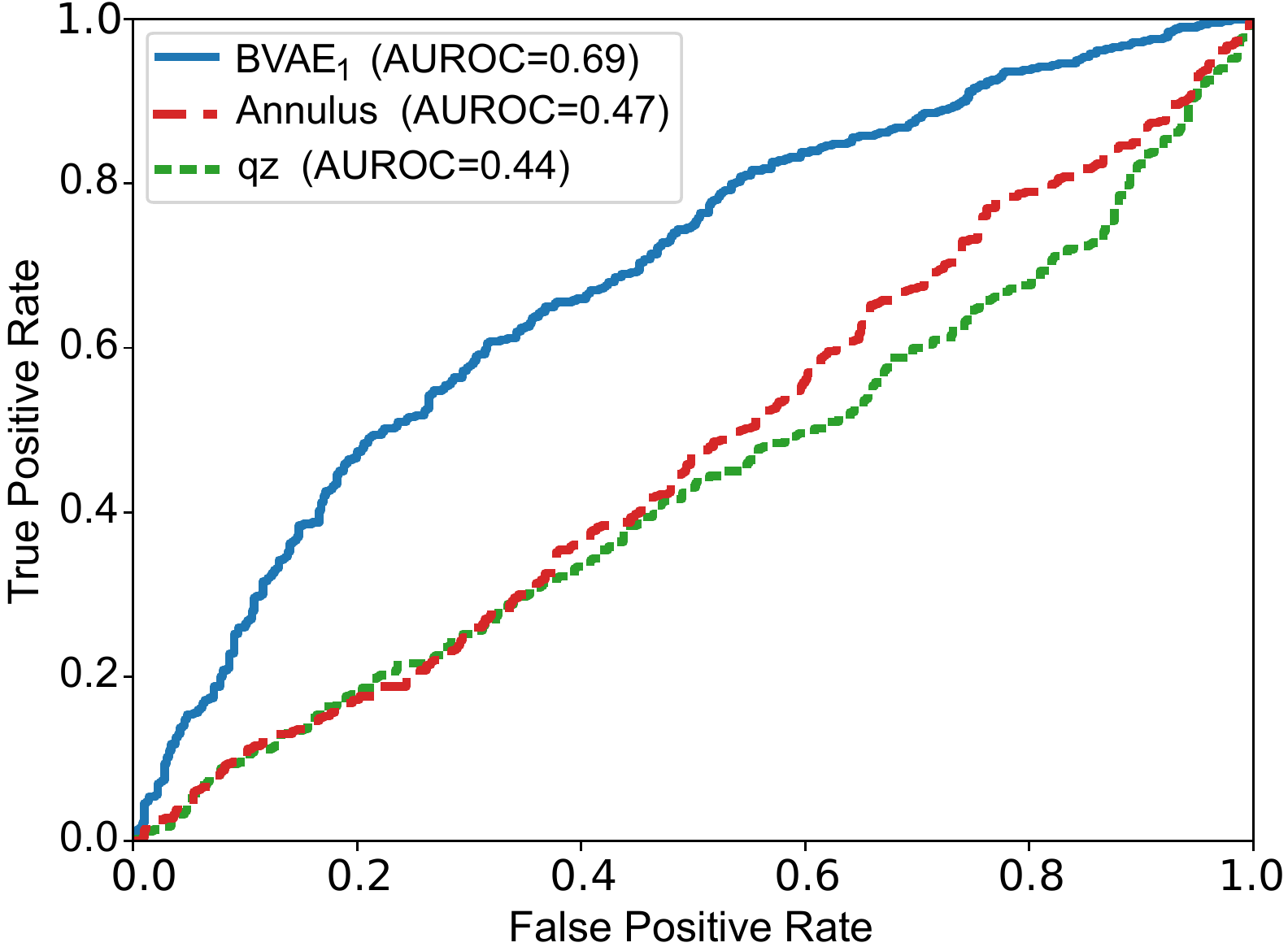}}
    \caption{(Top row) ROC curves (left) and histograms of \texttt{LL} (middle) and  \texttt{BVAE$_2$} (right) scores on the FashionMNIST vs.\ MNIST task; \texttt{BVAE$_2$} separates in-distribution and OoD data much more clearly than \texttt{LL}. (Bottom row) ROC curves for the SVHN vs.\ CIFAR10 (left) and latent space (right) task. }
    \label{fig:curves}
\end{figure*}
\subsection{Out-of-Distribution Detection in Latent Space}
\label{subsec:exp_latent}
While input space OoD detection is well-studied, \emph{latent space} OoD detection has only recently been identified as a critical open problem \citep{griffiths2017,gomez2018,mahmood2019,alperstein2019} (see also \cref{sec:intro}).
Thus, there is a lack of suitable experimental benchmarks, making a quantitative evaluation challenging.
A major issue in designing benchmarks based on commonly-used datasets such as MNIST is that it is unclear how to obtain \emph{ground truth labels} for which latent points are OoD and which are not, as we require OoD labels for \emph{all possible} latent points $\z^* \in \sZ$, not just for those corresponding to inputs $\x^*$ from the given dataset.
As a \emph{first step} towards facilitating a systematic empirical evaluation of latent space OoD detection techniques, we propose the following experimental protocol.
We use the \texttt{BVAE$_1$} variant (see \cref{subsec:exp_input}), as latent space detection does not require encoder robustness.
We train the model on FashionMNIST (or potentially any other dataset), and then sample $N = 10,000$ latent test points $\z^*$ from the Gaussian $\mathcal{N}(\mathbf{0}, b \cdot \sI_d)$ where $b \in \sR^+$ (we use $b = 10,000$), following \citet{mahmood2019}.
Since there do not exist ground truth labels for which latent points $\z^*$ are OoD or not, we compute a classifier-based \emph{OoD proxy score} (to be detailed below) for each of the $N$ latent test points and then simply \emph{define} the $N/2$ latents with the lowest scores to be in-distribution, and all others to be OoD.

To this end, we train an ensemble \citep{lakshminarayanan2017} of $J$ convolutional NN classifiers with parameters $\mathbf{W} = \{\w_j\}_{j=1}^J$ on FashionMNIST.
We then approximate the novelty score for discriminative models proposed by \citet{houlsby2011}, i.e.,
$Q_\mathbf{W}[\x^*] = \sH\left( \frac{1}{J} \sum_{\w \in \mathbf{W}} p(y|\x^*, \w) \right) - \frac{1}{J} \sum_{\w \in \mathbf{W}} \sH(p(y|\x^*,\w))$%
, where the first term is the entropy of the mixture $\frac{1}{J} \sum_{\w \in \mathbf{W}} p(y|\x^*, \w)$ of categorical distributions $p(y|\x^*, \w)$ (which is again categorical with averaged probits), and the second term is the average entropy of the predictive class distribution of the classifier with parameters $\w$.
Alternatively, one could also use the closely related OoD score $\sum_{\w \in \mathbf{W}} \KL(p(y|\x^*, \w) \| p(y|\x^*, \sD))$ of \citet{lakshminarayanan2017}.
Since $Q_\mathbf{W}[\x^*]$ requires a test input $\x^*$, and we only have the latent code $\z^*$ corresponding to $\x^*$ in our setting, we instead consider the \emph{expected} novelty under the mixture decoding distribution $p(\x|\z^*,\sD)$,
$\E_{p(\x|\z^*,\sD)}[Q_\mathbf{W}[\x]] \simeq \frac{1}{L} \sum_{l=1}^L Q_\mathbf{W}[\x_l]$ with $\x_l \sim p(\x|\z^*,\sD)$.
In practice, we use an ensemble of $J=5$ classifiers and $L=32$ input samples for the expectation.
We compare the \texttt{BVAE$_1$} model with our \emph{expected disagreement} score $D_\Theta[\z^*]$ (see \cref{subsec:ood_input}, with $N=32$ samples) against two baselines (which are the only existing methods we are aware of):
(a)~The distance of $\z^* \in \sR^d$ to the spherical annulus of radius $\sqrt{d-1}$, which is where most probability mass lies under our prior $\mathcal{N}(\mathbf{0}, \sI_d)$ (\texttt{Annulus}) \citep{alperstein2019}, and
(b)~the log-probability of $\z^*$ under the \emph{aggregated posterior} of the training data in latent space $q(\z) = \frac{1}{N} \sum_{\x \in \sD} q(\z|\x, \phi)$, i.e., a uniform mixture of $N$ Gaussians in our case (\texttt{qz})\footnote{For efficiency, we only consider the $100$ nearest neighbors (found by a $100$-NN model) of a latent test point $\z^*$ for computing this log-probability \citep{mahmood2019}.} \citep{mahmood2019}.
\cref{fig:curves} (bottom right) shows that our method significantly outperforms the two baselines on this task.

\section{Related Work}
\label{sec:related_work}
\textbf{Supervised/Discriminative OoD detection methods.}
Most existing OoD detection approaches are \emph{task-specific} in that they are applicable within the context of a given prediction task.
As described in \cref{sec:intro}, these approaches train a \emph{deep discriminative model} in a \emph{supervised} fashion using the given labels.
To detect outliers w.r.t.\ the target task, such approaches typically rely on some sort of \emph{confidence score} to decide on the reliability of the prediction, which is either produced by modifying the model and/or training procedure, or computed/extracted post-hoc from the model and/or predictions
\citep{an2015,solch2016,hendrycks2016,liang2017,hendrycks2018,shafaei2018,devries2018,sricharan2018,ahmed2019}.
Alternatively, some methods use predictive uncertainty estimates for OoD detection \citep{gal2016,lakshminarayanan2017,malinin2018,osawa2019,ovadia2019} (cf.\ \cref{subsec:motivation}).
The main drawback of such approaches is that discriminatively trained models by design discard all input features which are not informative about the specific prediction task at hand, such that information that is relevant for general OoD detection might be lost.
Thus, whenever the task changes, the predictive (and thus OoD detection) model must be re-trained from scratch, even if the input data remains the same.

\textbf{Unsupervised/Generative OoD detection methods.}
\emph{task-agnostic} OoD detection methods solely use the inputs for the \emph{unsupervised} training of a \emph{DGM} to capture the data distribution, which makes them independent of any prediction task and thus more general.
Only a few recent works fall into this category.
\citet{ren2019} propose to correct the likelihood $\log p(\x^* | \theta)$ for confounding general population level background statistics captured by a background model $p(\x^*|\theta_0)$, resulting in the score $\log p(\x^* | \theta) - \log p(\x^* | \theta_0)$.
The background model $p(\x^*|\theta_0)$ is in practice trained by perturbing the data $\sD$ with noise to corrupt its semantic structure, i.e., by sampling input dimensions i.i.d.\ from a Bernoulli distribution with rate $\mu \in [0.1,0.2]$ and replacing their values by uniform noise, e.g. $x_i \sim \gU\{0,\ldots,255\}$ for images.
\citet{choi2018} propose to use an ensemble \citep{lakshminarayanan2017} of independently trained likelihood-based DGMs (i.e., with random parameter initializations and random data shuffling) to approximate the \emph{Watanabe-Akaike Information Criterion} (WAIC) \citep{watanabe2010} $\E_{p(\theta|\sD)}[\log p(\x^*|\theta)] - \text{Var}_{p(\theta|\sD)}[\log p(\x^*|\theta)]$, which provides an asymptotically correct likelihood estimate between the training and test set expectations (however, assuming a \emph{fixed} underlying data distribution).
Finally, \citet{nalisnick2019} propose to account for the \emph{typicality} of $\x^*$ via the score $\left| \log p(\x^*|\theta) - \frac{1}{N} \sum_{\x \in \sD} \log p(\x|\theta) \right|$, although they focus on \emph{batches} of test inputs instead of single inputs.

\textbf{Latent space OoD detection.}
Many important problems in science and engineering involve optimizing an expensive black-box function over a highly structured (i.e., discrete or non-Euclidian) input space, e.g.\ graphs, sequences, or sets.
Such problems are often tackled using \emph{Bayesian optimization} (BO), which is an established framework for sample-efficient black-box optimization that has, however, mostly focused on continuous input spaces \cite{brochu2010tutorial,shahriari2015taking}.
To extend BO to structured input spaces, recent works either designed dedicated models and acquisition procedures to optimize structured functions in input space directly \cite{baptista2018bayesian,kim2019bayesian,daxberger2020mixed,oh2019combinatorial}, or instead train a deep generative model (e.g.\ a VAE) to map the structured input space onto a continuous latent space, where the optimization can then be performed using standard continuous BO techniques \citep{gomez2018}.
Despite recent successes of the latter so-called \emph{latent space optimization} approach in areas such as automatic chemical design and automatic machine learning \citep{gomez2018,kusner2017,nguyen2016synthesizing,luo2018,lu2018,jin2018junction,tripp2020sample}, it often progresses into regions of the latent space too far away from the training data, yielding meaningless or even invalid inputs.
A few recent works have tried to detect/avoid the progression into out-of-distribution regions in latent space, by either designing the generative model to produce valid inputs \citep{kusner2017,dai2018syntax}, or by explicitly constraining the optimization to stay in-distribution, which is quantified using certain proxy metrics \citep{griffiths2017,mahmood2019}.

\textbf{Bayesian DGMs.}
Only a few works have tried to do Bayesian inference in DGMs, none of which addresses OoD detection.
While \citet{kingma2013} describe how to do VI over the decoder parameters of a VAE (see their Appendix F), this is neither motivated nor empirically evaluated.
\citet{hernandez2016importance} do mean-field Gaussian VI over the encoder and decoder parameters of an importance-weighted autoencoder \cite{burda2015} to increase model flexibility and improve generalization performance.
\citet{nguyen2017variational} do mean-field Gaussian VI over the decoder parameters of a VAE to enable continual learning.
\citet{saatci2017bayesian} use SG-MCMC to sample the parameters of a generative adversarial network \citep{goodfellow2014generative} to increase model expressiveness.
\citet{gong2019icebreaker} use SG-MCMC to sample the decoder parameters of a VAE for feature-wise active learning.

\section{Conclusion}
We proposed an effective method for unsupervised out-of-distribution detection, both in input space and in latent space, which uses information-theoretic metrics based on the posterior distribution over the parameters of a deep generative model (in particular a VAE).
In the future, we want to explore extensions to
other approximate inference techniques (e.g. variational inference \cite{blei2017}), and
to other deep generative models (e.g., flow-based \cite{kingma2018} or auto-regressive \cite{van2016} models).
Finally, we hope that this paper will inspire many follow-up works that will (a) develop further benchmarks and methods for the underappreciated yet critical problem of \emph{latent space} OoD detection, and (b) further explore the described paradigm of information-theoretic OoD detection, which might be a promising approach towards the grand goal of making deep neural networks more reliable and robust.

\section*{Broader Impact}
A sophisticated out-of-distribution detection mechanism will be a critical component of any machine learning pipeline deployed in a safety-critical application domain, such as healthcare or autonomous driving.
Current algorithms are not able to identify scenarios in which they ought to fail, which is a major shortcoming, as that can lead to fatal decisions when deployed in decision-making pipelines.
This severely limits the applicability of current state-of-the-art methods to such application domains, which substantially hinders the generally wide potential benefits that machine learning could have on society.
We thus envision our approach to be beneficial and impactful in bringing deep neural networks and similar machine learning methods to such safety-critical applications.
That being said, failure of our system could indeed lead to sub-optimal and potentially fatal decisions.

\bibliographystyle{abbrv}
\bibliography{bibliography}

\clearpage
\appendix
\section{Stochastic Gradient Markov Chain Monte Carlo (SG-MCMC)}
\label{sec:sgmcmc}
To generate samples $\theta \sim p(\theta|\sD)$ of parameters $\theta$ of a DNN, one typically uses stochastic gradient MCMC methods such as stochastic gradient Hamiltonian Monte Carlo (SGHMC).
In particular, consider the posterior distribution $p(\theta | \sD) \propto \exp(-U(\theta, \sD))$ with potential energy function $U(\theta, \sD) = - \log p(\sD, \theta) = -\log(p(\sD | \theta) p(\theta)) = -\sum_{\x \in \sD} \log p(\x | \theta) - \log p(\theta)$ induced by the prior $p(\theta)$ and marginal log-likelihood $\log p(\x|\theta)$.
Hamiltonian Monte Carlo (HMC) \citep{duane1987,betancourt2017} is a method that generates samples $\theta \sim p(\theta|\sD)$ to efficiently explore the parameter space by simulating Hamiltonian dynamics, which involves evaluating the gradient $\nabla_\theta U(\theta)$ of $U$.
However, computing this gradient requires examining the entire dataset $\sD$ (due to the summation of the log-likelihood over all $\x \in \sD$), which might be prohibitively costly for large datasets.
To overcome this, \cite{chen2014} proposed SGHMC as a scalable HMC variant based on a noisy, unbiased gradient estimate $\nabla_\theta U(\theta, \sM) \simeq \nabla_\theta U(\theta, \sD)$ computed on a minibatch $\sM$ of points sampled uniformly at random from $\sD$ (i.e., akin to minibatch-based optimization algorithms such as stochastic gradient descent), i.e.,
\begin{flalign}
\label{eq:grad_U_M}
\hspace{-3.0mm}
  \nabla_\theta U(\theta, \sM) \hspace{-1mm}= \hspace{-1mm}- \textstyle\frac{|\sD|}{|\sM|} \textstyle\sum_{\x \in \sM} \hspace{-1mm} \nabla_\theta \log p(\x | \theta) \hspace{-0.5mm}- \hspace{-0.5mm}\nabla_\theta \log p(\theta) .\hspace{-2mm}
\end{flalign}
\section{Information-theoretic Perspective on the Proposed Disagreement Score}
\label{sec:info_perspective}
\subsection{An Information-theoretic Perspective}
\label{subsec:info_theory_perspective}
Expanding on \cref{subsec:ood_input}, we now provide a more principled justification for the disagreement score $D_\Theta[\x^*]$ in \cref{eq:D_ess}, which induces an \emph{information-theoretic perspective} on OoD detection and reveals an intriguing connection to \emph{active learning}.
Assume that given training data $\sD$ and a prior distribution $p(\theta)$ over the DGM parameters $\theta$, we have inferred a \emph{posterior distribution}
\begin{equation}
    p(\theta|\sD) = \frac{p(\sD|\theta) p(\theta)}{\int p(\sD|\theta) p(\theta) d\theta} = \frac{p(\sD|\theta)}{p(\sD)} p(\theta)  
\end{equation}
over $\theta$.
Then, for a given input $\x^*$, the score $D_\Theta[\x^*]$ quantifies \emph{how much the posterior $p(\theta|\sD)$ would change} if we were to add $\x^*$ to $\sD$ and then infer the \emph{augmented posterior}
\begin{equation}
    p(\theta|\sD^*) = \frac{p(\x^*| \theta) p(\theta|\sD)}{\int p(\x^*| \theta) p(\theta|\sD) d\theta} = \frac{p(\x^*|\theta)}{p(\x^*|\sD)} p(\theta|\sD)  
\end{equation}
based on this new training set $\sD^* = \sD \cup \{\x^*\}$.
To see this, first note that this change in the posterior is quantified by the \emph{normalized likelihood} $\frac{p(\x^*|\theta)}{p(\x^*|\sD)}$, such that
models $\theta$ under which $\x^*$ is more (less) likely -- relative to all other models -- will have a higher (lower) probability under the updated posterior $p(\theta|\sD^*)$.
Now, given the samples $\{ \theta_m \}_{m=1}^M$ of the old posterior $p(\theta|\sD)$, the normalized likelihood $\frac{p(\x^*|\theta)}{p(\x^*|\sD)}$ for a given model $\theta$ is proportional to $w_\theta$ in \cref{eq:D_ess}, i.e.,
\begin{flalign}
\label{eq:smc}
\hspace{-4mm}
    \textstyle\frac{p(\x^*|\theta)}{p(\x^*|\sD)} \overset{(\ref{eq:expected_lik})}{=} \textstyle\frac{p(\x^*|\theta)}{\E_{p(\theta|\sD)}[p(\x^*|\theta)]} \simeq \textstyle\frac{p(\x^*|\theta)}{\frac{1}{M} \sum_{\theta \in \Theta} p(\x^*|\theta)} \overset{(\ref{eq:D_ess})}{=} M w_\theta\ .
\hspace{-3mm}
\end{flalign}
Thus, $w_\theta$ intuitively measures the \emph{relative usefulness} of $\theta$ for describing the new posterior $p(\theta|\sD^*)$.
More formally, the $[w_\theta]_{\theta \in \Theta}$ correspond to the \emph{importance weights} of the samples $\theta \in \Theta$ drawn from the \emph{proposal} distribution $p(\theta|\sD)$ for an importance sampling-based Monte Carlo approximation of an expectation w.r.t.\ the \emph{target} distribution $p(\theta|\sD^*)$,
\begin{flalign}
\label{eq:E_new}
\hspace{-3mm}
    \E_{p(\theta|\sD^*)}[f(\theta)] \overset{(\ref{eq:smc})}{\simeq} \E_{p(\theta|\sD)}[M w_\theta f(\theta)] \simeq \textstyle\sum_{\theta \in \Theta} w_\theta f(\theta)
\hspace{-2mm}
\end{flalign}
for any function $f: \Theta \rightarrow \sR$.
The score $D_\Theta[\x^*]$ in \cref{eq:D_ess} is a widely used measure of the efficiency of the estimator in \cref{eq:E_new}, known as the \emph{effective sample size} (ESS) of $\{ \theta_m \}_{m=1}^M$ \cite{martino2017}.
It quantifies how many i.i.d.\ samples drawn from the target posterior $p(\theta|\sD^*)$ are equivalent to the $M$ samples $\theta \in \Theta$ drawn from the proposal posterior $p(\theta|\sD)$ and weighted according to $w_\theta$, and thus indeed measures the \emph{change in distribution} from $p(\theta|\sD)$ to $p(\theta|\sD^*)$.
Equivalently, $D_\Theta[\x^*]$ can be viewed as quantifying the \emph{informativeness} of $\x^*$ for updating the DGM parameters $\theta$ to the ones capturing the true density.%
\footnote{\label{footnote}This connection is described in further detail in \cref{subsec:justification}.}

The OoD detection mechanism described in \cref{subsec:ood_input} can thus be intuitively summarised as follows:
\emph{In-distribution} inputs $\x^*$ are similar to the data points already in $\sD$ and thus \emph{uninformative} about the model parameters $\theta$, inducing \emph{small} change in distribution from $p(\theta|\sD)$ to $p(\theta|\sD^*)$, resulting in a \emph{large} ESS $D_\Theta[\x^*]$.
Conversely, \emph{OoD} inputs $\x^*$ are very different from the previous observations in $\sD$ and thus \emph{informative} about the model parameters $\theta$, inducing \emph{large} change in the posterior, resulting in a \emph{small} ESS $D_\Theta[\x^*]$.

Finally, this information-theoretic perspective on OoD detection reveals a close relationship to \emph{information-theoretic active learning} \cite{mackay1992,houlsby2011}.
There, the same notion of \emph{informativeness} (or, equivalently, \emph{disagreement}) is used to quantify the \emph{novelty} of an input $\x^*$ to be added to the data $\sD$, aiming to maximally improve the estimate of the model parameters $\theta$ by maximally \emph{reducing the entropy / epistemic uncertainty} in the posterior $p(\theta|\sD)$.
This is further justified in the next section.

\subsection{Further Justification}
\label{subsec:justification}
In \cref{subsec:info_theory_perspective}, we mentioned that the disagreement score $D_\Theta[\x^*]$ defined in \cref{eq:D_ess} can be viewed as quantifying the \emph{informativeness} of the input $\x^*$ for updating the DGM parameters $\theta$ to the ones capturing the true density, yielding an information-theoretic perspective on OoD detection and revealing a close relationship to \emph{information-theoretic active learning} \cite{mackay1992}.
While this connection intuitively sensible, we now further describe and justify it.

In the paradigm of \emph{active learning}, the goal is to iteratively select inputs $\x^*$ which improve our estimate of the model parameters $\theta$ as rapidly as possible, in order to obtain a decent estimate of $\theta$ using as little data as possible, which is critical in scenarios where obtaining training data is expensive (e.g. in domains where humans or costly simulations have to be queried to obtain data, which includes many medical or scientific applications).
The main idea of \emph{information-theoretic} active learning is to maintain a posterior distribution $p(\theta|\sD)$ over the model parameters $\theta$ given the training data $\sD$ observed thus far, and to then select the new input $\x^*$ based on its \emph{informativeness} about the distribution $p(\theta|\sD)$, which is measured by the \emph{change in distribution} between the current $p(\theta|\sD)$ posterior and the updated posterior $p(\theta|\sD^*)$ with $\sD^* = \sD \cup \{\x^*\}$.
This change in the posterior distribution can, for example, be quantified by the cross-entropy or KL divergence between $p(\theta|\sD)$ and $p(\theta|\sD^*)$ \cite{mackay1992}, or by the decrease in entropy between $p(\theta|\sD)$ and $p(\theta|\sD^*)$ \cite{houlsby2011}.

Intriguingly, while the problems of active learning and out-of-distribution detection have clearly distinct goals, they are fundamentally related in that they both critically rely on a reliable way to quantify how \emph{different} an input $\x^*$ is from the training data $\sD$ (or, put differently, how \emph{novel} or \emph{informative} $\x^*$ is).
While in active learning, we aim to identify the input $\x^*$ that is maximally \emph{different} (or \emph{novel} / \emph{informative}) in order to best improve our estimate of the model parameters by adding $\x^*$ to the training dataset $\sD$, in out-of-distribution detection, we aim to classify a given input $\x^*$ as either in-distribution or OoD based on how \emph{different} (or \emph{novel} / \emph{informative}) it is.
This naturally suggests the possibility of leveraging methods to quantify the \emph{novelty} / \emph{informativeness} of an input $\x^*$ developed for one problem, and apply it to the other problem.
However, most measures used in active learning are designed for \emph{continuous} representations of the distributions $p(\theta|\sD)$ and $p(\theta|\sD^*)$, and are not directly applicable in our setting where $p(\theta|\sD)$ and $p(\theta|\sD)$ are represented by a \emph{discrete} set of samples $\Theta$.

That being said, $D_\Theta[\x^*]$ can indeed be viewed as quantifying the \emph{change in distribution} between the sample-based representations of $p(\theta|\sD)$ and $p(\theta|\sD^*)$ induced by $\x^*$ (and thus the \emph{informativeness} of $\x^*$), revealing a link to information-theoretic active learning.
In particular, \cite{martino2017} show that $D_\Theta[\x^*]$ (which corresponds to the \emph{effective sample size}, as described in \cref{subsec:info_theory_perspective}) is closely related to the \emph{Euclidean distance} between the vector of importance weights $\mathbf{w} = [w_\theta]_{\theta \in \Theta}$ and the vector $\mathbf{w}^* = [\frac{1}{M}]_{\theta \in \Theta}$ of probabilities defining the \emph{discrete uniform probability mass function}, i.e.,
\begin{equation}
    \|\mathbf{w} - \mathbf{w}^*\|_2 = \sqrt{\frac{1}{D_\Theta[\x^*]} - \frac{1}{M}} \quad \Longleftrightarrow \quad D_\Theta[\x^*] = \frac{1}{\|\mathbf{w} - \mathbf{w}^*\|_2^2 + \frac{1}{M}}
\end{equation}
such that \emph{maximizing} the score $D_\Theta[\x^*]$ is equivalent to \emph{minimizing} the Euclidian distance $\|\mathbf{w} - \mathbf{w}^*\|_2$.
Now, since
\begin{equation}
\label{eq:new_post}
    p(\theta|\sD^*) = \frac{p(\x^*| \theta) p(\theta|\sD)}{\int p(\x^*| \theta) p(\theta|\sD) d\theta} = \frac{p(\x^*|\theta)}{p(\x^*|\sD)} p(\theta|\sD) \overset{(\ref{eq:smc})}\simeq M w_\theta p(\theta|\sD)\ ,
\end{equation}
we observe that for a given model $\theta \in \Theta$, the posterior $p(\theta|\sD^*)$ is equal to $p(\theta|\sD)$ if and only if $M w_\theta = 1 \Longleftrightarrow w_\theta = \frac{1}{M}$, such that $p(\theta|\sD^*)$ is equal to $p(\theta|\sD)$ for \emph{all} models $\theta \in \Theta$ if and only if the weight vector $\mathbf{w} = [w_\theta]_{\theta \in \Theta}$ is equal to the vector $\mathbf{w}^* = [\frac{1}{M}]_{\theta \in \Theta}$ defining the discrete uniform probability mass function (pmf), in which case their Euclidian distance is minimized at $\|\mathbf{w} - \mathbf{w}^*\|_2 = 0$.
As a result, the new posterior $p(\theta|\sD^*)$ is identical to the previous posterior $p(\theta|\sD)$ over the models $\theta \in \Theta$ (i.e., the change in the posterior is \emph{minimized}) if and only if the score $D_\Theta[\x^*]$ is \emph{maximized} to be $D_\Theta[\x^*] = M$.
Conversely, the Euclidean distance is \emph{maximized} at $\|\mathbf{w} - \mathbf{w}^*\|_2 = \sqrt{(1 - \frac{1}{M})}$ if and only if the weight vector is $\mathbf{w} = [0, \ldots, 0, 1, 0, \ldots, 0]$, in which case the new posterior is $p(\theta|\sD^*) = 0$ for all $M-1$ models $\theta$ for which $w_\theta = 0$, and $p(\theta|\sD^*) = M p(\theta|\sD)$ for the single model $\theta$ for which $w_\theta = 1$.
Thus, the change between the new and previous posterior over the models $\theta \in \Theta$ is \emph{maximized} if and only if the score $D_\Theta[\x^*]$ is \emph{minimized} to be $D_\Theta[\x^*] = 1$.

Finally, we observe that the notion of \emph{change in distribution} for sample-based representations of posteriors captured by the Euclidian distance $\|\mathbf{w} - \mathbf{w}^*\|_2$ described above is \emph{closely related} to the notion of \emph{change in distribution} for continuous posterior representations.
To see this, consider the \emph{Kullback-Leibler (KL) divergence}, which is an information-theoretic measure for the discrepancy between distributions commonly used in information-theoretic active learning, defined as
\begin{equation}
\label{eq:kl}
    \KL[p(\theta|\sD) \| p(\theta|\sD^*)] = \int p(\theta|\sD) \log \frac{p(\theta|\sD)}{p(\theta|\sD^*)} d\theta \ .
\end{equation}
We now show that \emph{maximizing} our proposed OoD detection score $D_\Theta[\x^*]$ is equivalent to \emph{minimizing} the KL divergence $\KL[p(\theta|\sD) \| p(\theta|\sD^*)]$ between the previous posterior $p(\theta|\sD)$ and the new posterior $p(\theta|\sD^*)$, and vice versa, which is formalized in \cref{prop} below.
This provides further evidence for the close connection between our proposed OoD detection approach and information-theoretic principles, and suggests that information-theoretic measures such as the KL divergence can be also used for OoD detection, yielding the paradigm of \emph{information-theoretic out-of-distribution detection}.
\begin{proposition}
\label{prop}
    Assume that the weights $w_\theta$ have some minimal, arbitrarily small, positive value $\varepsilon > 0$, i.e., $w_\theta > \varepsilon, \forall \theta \in \Theta$.
    Also, assume that the KL divergence $\KL[p(\theta|\sD) \| p(\theta|\sD^*)]$ is approximated based on a set $\Theta = \{\theta_m\}_{m=1}^M$ of samples $\theta_m \sim p(\theta|\sD)$.
    Then, an input $\x^* \in \sX$ is a \emph{maximizer} of $D_\Theta[\x^*]$ if and only if it is a \emph{minimizer} of $\KL[p(\theta|\sD) \| p(\theta|\sD^*)]$.
    Furthermore, an input $\x^* \in \sX$ is a \emph{minimizer} of $D_\Theta[\x^*]$ if and only if it is a \emph{maximizer} of $\KL[p(\theta|\sD) \| p(\theta|\sD^*)]$.
    Formally,
    \begin{align}
    \label{eq:max_min}
        \operatorname{arg\,max}_{\x^* \in \sX} D_\Theta[\x^*] &= \operatorname{arg\,min}_{\x^* \in \sX} \KL[p(\theta|\sD) \| p(\theta|\sD^*)]\ ,\\
    \label{eq:min_max}
        \operatorname{arg\,min}_{\x^* \in \sX} D_\Theta[\x^*] &= \operatorname{arg\,max}_{\x^* \in \sX} \KL[p(\theta|\sD) \| p(\theta|\sD^*)]\ .
    \end{align}
\end{proposition}
\begin{proof}
Reformulating the KL divergence in \cref{eq:kl} and approximating it via our set $\Theta$ of posterior samples, we obtain
\begin{align}
\nonumber
    \KL[p(\theta|\sD) \| p(\theta|\sD^*)] &\overset{(\ref{eq:kl})}{=} \int p(\theta|\sD) \log \frac{p(\theta|\sD)}{p(\theta|\sD^*)} d\theta\\
\nonumber
    &= -\int p(\theta|\sD) \log \frac{p(\theta|\sD^*)}{p(\theta|\sD)} d\theta\\
\nonumber
    &\overset{(\ref{eq:new_post})}{=} -\int p(\theta|\sD) \log \frac{p(\x^*|\theta)}{p(\x^*|\sD)} d\theta\\
\nonumber
    &\overset{(\ref{eq:expected_lik})}{=} -\int p(\theta|\sD) \log \frac{p(\x^*|\theta)}{\E_{p(\theta|\sD)}[p(\x^*|\theta)]} d\theta\\
\nonumber
    &= -\E_{p(\theta|\sD)}\left[ \log \frac{p(\x^*|\theta)}{\E_{p(\theta|\sD)}[p(\x^*|\theta)]} \right]\\
\nonumber
    &\simeq -\E_{p(\theta|\sD)}\left[ \log \frac{p(\x^*|\theta)}{\frac{1}{M} \sum_{\theta \in \Theta} p(\x^*|\theta)} \right]\\
\nonumber
    &\simeq -\frac{1}{M} \sum_{\theta \in \Theta} \log \frac{p(\x^*|\theta)}{\frac{1}{M} \sum_{\theta \in \Theta} p(\x^*|\theta)}\\
\label{eq:kl_smp}
    &\overset{(\ref{eq:D_ess})}{=} -\frac{1}{M} \sum_{\theta \in \Theta} \log M w_\theta\ .
\end{align}
To see that \cref{eq:max_min} holds, observe that the sample-based approximation of $\KL[p(\theta|\sD) \| p(\theta|\sD^*)]$ in \cref{eq:kl_smp} is indeed \emph{minimized} to be $\KL[p(\theta|\sD) \| p(\theta|\sD^*)] = 0$ if and only if the weight vector $\mathbf{w} = [w_\theta]_{\theta \in \Theta}$ is equal to the vector $\mathbf{w} = [\frac{1}{M}]_{\theta \in \Theta}$ defining the discrete uniform pmf (and thus if and only if the score $D_\Theta[\x^*]$ is \emph{maximized} to be $D_\Theta[\x^*] = M$), as then 
\begin{equation}
    \KL[p(\theta|\sD) \| p(\theta|\sD^*)] \overset{(\ref{eq:kl_smp})}{\simeq} -\frac{1}{M} \sum_{\theta \in \Theta} \log M w_\theta = -\frac{1}{M} M \log M \frac{1}{M} = -\log 1 = 0\ .
\end{equation}
We now show \cref{eq:min_max}.
To see why we need the assumption that the weights $w_\theta$ have some minimal, arbitrarily small, positive value $\varepsilon > 0$, i.e., $w_\theta \geq \varepsilon, \forall \theta \in \Theta$, consider the unconstrained case, where we know that $D_\Theta[\x^*]$ is \emph{minimized} to be $D_\Theta[\x^*] = 1$ if and only if the weight vector $\mathbf{w} = [w_\theta]_{\theta \in \Theta}$ is equal to the vector $\mathbf{w} = [0, \ldots, 0, 1, 0, \ldots, 0]$.
While this weight vector indeed \emph{maximizes} the sampling-based approximation of the KL divergence to be $\KL[p(\theta|\sD) \| p(\theta|\sD^*)] = \infty$,
\begin{equation}
    \KL[p(\theta|\sD) \| p(\theta|\sD^*)] \overset{(\ref{eq:kl_smp})}{\simeq} -\frac{1}{M} \sum_{\theta \in \Theta} \log M w_\theta = -\frac{1}{M} [\log M + (M-1) \log 0] = \infty\ ,
\end{equation}
this maximizer is not unique, as any other weight vector containing at least one weight of $w_\theta = 0$ equally achieves the maximum of $\KL[p(\theta|\sD) \| p(\theta|\sD^*)] = \infty$.
In theory, the KL divergence can thus not distinguish between weight vectors with different numbers of zero entries (i.e., different $\ell_0$-norms $\|\mathbf{w}\|_0$), although these clearly define different degrees of change in the discrete posterior representation.
However, in practice, it is very unlikely to occur that any $w_\theta = 0$.
To obtain a unique maximizer of the KL divergence, we thus assume $w_\theta \geq \varepsilon, \forall \theta \in \Theta$ (where $\varepsilon > 0$ can be chosen to be arbitrarily small), in which case $\KL[p(\theta|\sD) \| p(\theta|\sD^*)]$ is maximized if and only if $\mathbf{w} = [\varepsilon, \ldots, \varepsilon, 1 - (M-1) \varepsilon, \varepsilon, \ldots, \varepsilon]$, in which case
\begin{equation}
\label{eq:kl_eps}
    \KL[p(\theta|\sD) \| p(\theta|\sD^*)] \simeq -\frac{1}{M} \sum_{\theta \in \Theta} \log M w_\theta = -\frac{1}{M} [\log M (1 - (M-1) \varepsilon) + (M-1) \log M \varepsilon] < \infty\ .
\end{equation}
The positivity assumption thus also ensures that the KL divergence remains bounded.
To see why $\mathbf{w} = [\varepsilon, \ldots, \varepsilon, 1 - (M-1) \varepsilon, \varepsilon, \ldots, \varepsilon]$ maximizes the KL divergence, consider the alternative vector $\mathbf{w} = [\varepsilon, \ldots, \varepsilon, \varepsilon + \delta, \varepsilon, \ldots, \varepsilon, 1 - (M-1) \varepsilon - \delta, \varepsilon, \ldots, \varepsilon]$ where any of the entries with minimal value $\varepsilon$ is increased by some arbitrarily small, positive $\delta > 0$, such that
\begin{equation}
\label{eq:kl_delta}
    \KL[p(\theta|\sD) \| p(\theta|\sD^*)] \simeq -\frac{1}{M} \sum_{\theta \in \Theta} \log M w_\theta -\frac{1}{M} [\log M (1 - (M-1) \varepsilon - \delta) + (M-2) \log M \varepsilon + \log M (\varepsilon + \delta) ]\ .
\end{equation}
To see that adding such a $\delta$ decreases the value of the KL divergence, observe that \cref{eq:kl_eps} and \cref{eq:kl_delta} yield
\begin{align*}
    \cancel{-\frac{1}{M}} [\log M (1 - (M-1) \varepsilon) + (M-1) \log M \varepsilon] &> \cancel{-\frac{1}{M}} [\log M (1 - (M-1) \varepsilon - \delta) + (M-2) \log M \varepsilon + \log M (\varepsilon + \delta) ]\\
    \log M (1 - (M-1) \varepsilon) + \cancel{(M-1)} \log M \varepsilon &< \log M (1 - (M-1) \varepsilon - \delta) + \cancel{(M-2) \log M \varepsilon} + \log M (\varepsilon + \delta)\\
    \log \cancel{M} (1 - (M-1) \varepsilon) + \log \cancel{M} \varepsilon &< \log \cancel{M} (1 - (M-1) \varepsilon - \delta) + \log \cancel{M} (\varepsilon + \delta)\\
    \log (1 - (M-1) \varepsilon) + \log \varepsilon &< \log (1 - (M-1) \varepsilon - \delta) + \log (\varepsilon + \delta) \\
    \cancel{\log} (1 - (M-1) \varepsilon) \varepsilon &< \cancel{\log} (1 - (M-1) \varepsilon - \delta) (\varepsilon + \delta) \\
    (1 - (M-1) \varepsilon) \varepsilon &< (1 - (M-1) \varepsilon - \delta) (\varepsilon + \delta) \\
    \varepsilon - (M-1) \varepsilon^2 &< (\varepsilon + \delta) - (M-1) \varepsilon (\varepsilon + \delta) - \delta (\varepsilon + \delta)\\
    \varepsilon - M \varepsilon^2 + \varepsilon^2 &< \varepsilon + \delta - (M-1) (\varepsilon^2 + \varepsilon \delta) - \varepsilon \delta - \delta^2\\
    \varepsilon - M \varepsilon^2 + \varepsilon^2 &< \varepsilon + \delta - [M \varepsilon^2 + M \varepsilon \delta - \varepsilon^2 - \varepsilon \delta] - \varepsilon \delta - \delta^2\\
    \cancel{\varepsilon} \cancel{- M \varepsilon^2} + \cancel{\varepsilon^2} &< \cancel{\varepsilon} + \delta \cancel{- M \varepsilon^2} - M \varepsilon \delta + \cancel{\varepsilon^2} + \cancel{\varepsilon \delta - \varepsilon \delta} - \delta^2\\
    0 &< \delta - M \varepsilon \delta - \delta^2\\
    0 &< 1 - M \varepsilon - \delta\\
    M \varepsilon &< 1 - \delta\\
    \varepsilon &< \frac{1 - \delta}{M}\ .
\end{align*}
The condition $\varepsilon < \frac{1 - \delta}{M}$ implies that the largest weight in $\mathbf{w}$, denoted by $w_{\theta_m}$, satisfies
\begin{align*}
    w_{\theta_m} &= 1 - (M-1) \varepsilon - \delta\\
    &> 1 - (M-1) \frac{1 - \delta}{M} - \delta\\
    &= 1 - \delta - \frac{M - M\delta - 1 + \delta}{M}\\
    &= \cancel{1 - \delta - 1 + \delta} + \frac{1 - \delta}{M}\\
    &= \frac{1 - \delta}{M}
\end{align*}
Thus, adding an arbitrarily small $\delta$ to one of the entries of $\mathbf{w}$ indeed decreases the value of the KL divergence, except when $\varepsilon = \frac{1 - \delta}{M}$, in which case the KL divergences remains the same.
However, in that case, the previously largest weight becomes $w_{\theta_m} = \frac{1 - \delta}{M} = \varepsilon$, yielding the weight vector $\mathbf{w} = [\frac{1 - \delta}{M}, \ldots, \frac{1 - \delta}{M}, \frac{1 - \delta}{M} + \delta, \frac{1 - \delta}{M}, \ldots, \frac{1 - \delta}{M}]$, which is close to the discrete uniform pmf and thus results in a KL divergence close to zero (which is thus not relevant for characterizing the \emph{maximizer} of the KL divergence).
We can analogously identify $\mathbf{w} = [\varepsilon, \ldots, \varepsilon, 1 - (M-1) \varepsilon, \varepsilon, \ldots, \varepsilon]$ to be a \emph{minimizer} of $D_\Theta[\x^*]$.

To conclude, since we can choose $\varepsilon$ to be arbitrarily small, it indeed holds that the KL divergence $\KL[p(\theta|\sD) \| p(\theta|\sD^*)]$ is \emph{minimized} if and only if the score $D_\Theta[\x^*]$ is \emph{maximized}, which is when the importance weight vector defines the discrete uniform pmf, i.e., $\mathbf{w} = [\frac{1}{M}]_{\theta \in \Theta}$. 
Moreover, the KL divergence $\KL[p(\theta|\sD) \| p(\theta|\sD^*)]$ is \emph{maximized} if and only if the score $D_\Theta[\x^*]$ is \emph{minimized}, which is when the importance weight vector is equal to $\mathbf{w} = [\varepsilon, \ldots, \varepsilon, 1 - (M-1) \varepsilon, \varepsilon, \ldots, \varepsilon]$.
\end{proof}
\section{Pseudocode of BVAE Training Procedure}
\label{sec:pseudocode_bvae_alt}
Pseudocode for training a \emph{Bayesian} VAE (for both variants 1 and 2, as described in \cref{sec:bvae}) is shown in \cref{alg:bvae2}, which is contrasted to the pseudocode for training a \emph{regular} VAE in \cref{alg:vae} , allowing for a direct comparison between the closely related training procedures.
In particular, in \cref{alg:bvae2}, the parts in {\color{violet} purple} correspond to parts that are \emph{different} from VAE training in \cref{alg:vae} and that apply to \emph{both} BVAE variants 1 and 2.
Furthermore, the parts in \cref{alg:bvae2} in {\color{blue} blue} correspond to BVAE \emph{variant 1 only}, while the parts in {\color{red} red} correspond to BVAE \emph{variant 2 only}.

I.e., the training procedure of BVAE \emph{variant 1} is described by the union of all black, {\color{violet} purple} and {\color{blue} blue} parts in \cref{alg:bvae2}, where the only difference to the regular VAE training procedure in \cref{alg:vae} is that an SG-MCMC sampler is used instead of an SGD optimizer for the decoder parameters $\theta$.
The training procedure of BVAE \emph{variant 2} is described by the union of all black, {\color{violet} purple} and {\color{red} red} parts in \cref{alg:bvae2}, where, in contrast to the regular VAE training procedure in \cref{alg:vae}, an SG-MCMC sampler is used instead of an SGD optimizer for \emph{both} the decoder parameters $\theta$ and the encoder parameters $\phi$.

For the BVAE training procedure in \cref{alg:bvae2}, we thus have to additionally specify the \emph{burn-in length} $B$, which denotes the number of samples to discard at the beginning before storing any samples, as well as the \emph{sample distance} $D$, which denotes the number of samples to discard in-between two subsequently stored samples (i.e., controlling the degree of correlation between the stored samples).
This results in a total of $M = (T-B) / D + 1$ samples for each sampling chain.

Regular VAE training in \cref{alg:vae} thus produces \emph{point estimates} $\theta_T$ for the decoder parameters and $\phi_T$ for the encoder parameters, while Bayesian VAE training produces a set $\Theta = \{\theta_B, \theta_{B+D}, \theta_{B+2D}, \ldots, \theta_T\}$ of \emph{posterior samples} $\phi_t \sim p(\phi|\sD)$ of decoder parameters, as well as either a \emph{point estimate} $\phi_T$ for the encoder parameters (in case of variant 1), or a set $\Phi = \{\phi_B, \phi_{B+D}, \phi_{B+2D}, \ldots, \phi_T\}$ of \emph{posterior samples} $\theta_t \sim p(\theta|\sD)$ of encoder parameters (in case of variant 2).

Note that just like the regular VAE training procedure in \cref{alg:vae}, the Bayesian VAE training procedure in \cref{alg:bvae2} can in practice be conveniently implemented by exploiting automatic differentiation tools commonly employed by modern deep learning frameworks.
Finally, as SG-MCMC methods are not much more expensive to run than stochastic optimization methods (i.e., both requiring a stochastic gradient step in every iteration, but SG-MCMC potentially requiring more iterations $T$ to generate $M$ diverse samples), training a Bayesian VAE is not significantly more expensive than training a regular VAE.

\begin{minipage}{0.49\textwidth}
\begin{algorithm}[H]
\caption{\emph{Regular} VAE Training}
\label{alg:vae}
\begin{algorithmic}
    \STATE {\bfseries In.} Dataset $\sD$, mini-batch size $|\sM|$, number of epochs $T$, generative model $p(\x, \z, \theta)$, inference model $q(\z|\x, \phi)$
    \STATE \phantom{$|$}
    \STATE Initialize $\phi_0$ and $\theta_0$
	\FOR{$t = 1, \ldots, T$}
	    \STATE Set $\hat{\phi}_0 = \phi_{t-1}$, $\hat{\theta}_0 = \theta_{t-1}$
    	\FOR{$b = 1, \ldots, \frac{|\sD|}{|\sM|}$}
            \STATE Sample minibatch $\sM \sim \sD$
            \STATE Update $\hat{\phi}_{b-1} \rightarrow \hat{\phi}_b$ via SGD \phantom{$\{$}
            \STATE \phantom{$|$}
            \STATE Update $\hat{\theta}_{b-1} \rightarrow \hat{\theta}_b$ via SGD
        \ENDFOR
        \STATE Set $\phi_t = \hat{\phi}_{\frac{|\sD|}{|\sM|}}$, $\theta_t = \hat{\theta}_{\frac{|\sD|}{|\sM|}}$
        \STATE \phantom{$($}
        \STATE \phantom{$\{$}
        \STATE 
    \ENDFOR
	\STATE {\bfseries Out.} Decoder $\theta_T$ and encoder $\phi_T$ \phantom{$\{$}
    \STATE \phantom{$|$}
\end{algorithmic} 
\end{algorithm}
\end{minipage}%
\hfill
\begin{minipage}{0.49\textwidth}
\begin{algorithm}[H]
\caption{{\color{violet} \emph{BVAE}} Training ({\color{blue} Variant 1} \& {\color{red} 2})}
\label{alg:bvae2}
\begin{algorithmic}
    \STATE {\bfseries In.} Dataset $\sD$, mini-batch size $|\sM|$, number of epochs~$T$, generative model $p(\x, \z, \theta)$, inference model $q(\z|\x, \phi)$, {\color{violet} burn-in length $B$, sample distance $D$}
    \STATE Initialize $\phi_0$ and $\theta_0$, {\color{violet} and $\Theta = \emptyset$} {\color{red} and $\Phi = \emptyset$}
	\FOR{$t = 1, \ldots, T$}
	    \STATE Set $\hat{\phi}_0 = \phi_{t-1}$, $\hat{\theta}_0 = \theta_{t-1}$
    	\FOR{$b = 1, \ldots, \frac{|\sD|}{|\sM|}$}
            \STATE Sample minibatch $\sM \sim \sD$
            \STATE Update $\hat{\phi}_{b-1} \rightarrow \hat{\phi}_b$ via $\{$ {\color{blue} SGD} $|$ {\color{red} SG-MCMC} $\}$ 
            \STATE Update $\hat{\theta}_{b-1} \rightarrow \hat{\theta}_b$ via {\color{violet} SG-MCMC}
        \ENDFOR
        \STATE Set $\phi_t = \hat{\phi}_{\frac{|\sD|}{|\sM|}}$, $\theta_t = \hat{\theta}_{\frac{|\sD|}{|\sM|}}$
        {\color{violet} 
        \IF{$t \geq B$ and $(t-B) \text{ mod } D = 0$}
            \STATE Add $\Theta = \Theta \cup \{\theta_t\}$ {\color{red} and $\Phi = \Phi \cup \{\phi_t\}$}
        \ENDIF}
    \ENDFOR
	\STATE {\bfseries Out.} Decoder {\color{violet} samples $\Theta$} and encoder $\{$ {\color{blue} $\phi_T$} $|$ {\color{red} samples $\Phi$} $\}$ 
\end{algorithmic} 
\end{algorithm}
\end{minipage}%
\section{Additional Plots for Experiments}
\label{sec:additional_exp}
We show additional plots for the experiments conducted in \cref{sec:experiments}.
\subsection{Score Histograms for SVHN vs.\ CIFAR10 Task}
To complement the histograms of scores in \cref{fig:curves} for the FashionMNIST vs.\ MNIST task, \cref{fig:hist-svhn} shows similar behaviour for the corresponding histograms for the SVHN vs.\ CIFAR10 task.
\begin{figure*}[ht!]
	\centering
	\raisebox{-0.46\height}{\includegraphics[trim={8mm 0 0 0},width=0.49\textwidth]{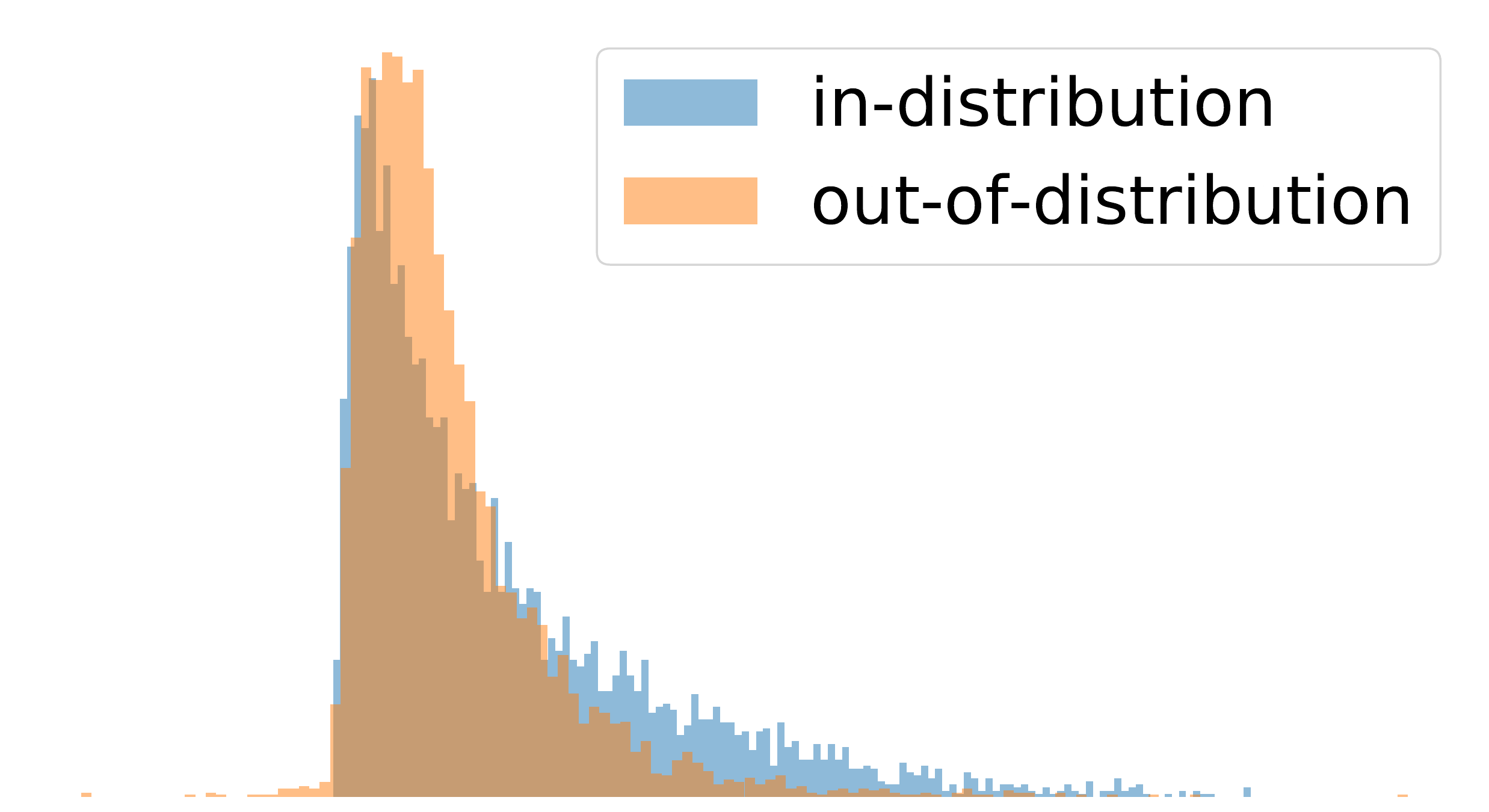}}
	\raisebox{-0.46\height}{\includegraphics[trim={12mm 0 14mm 0},clip,width=0.45\textwidth]{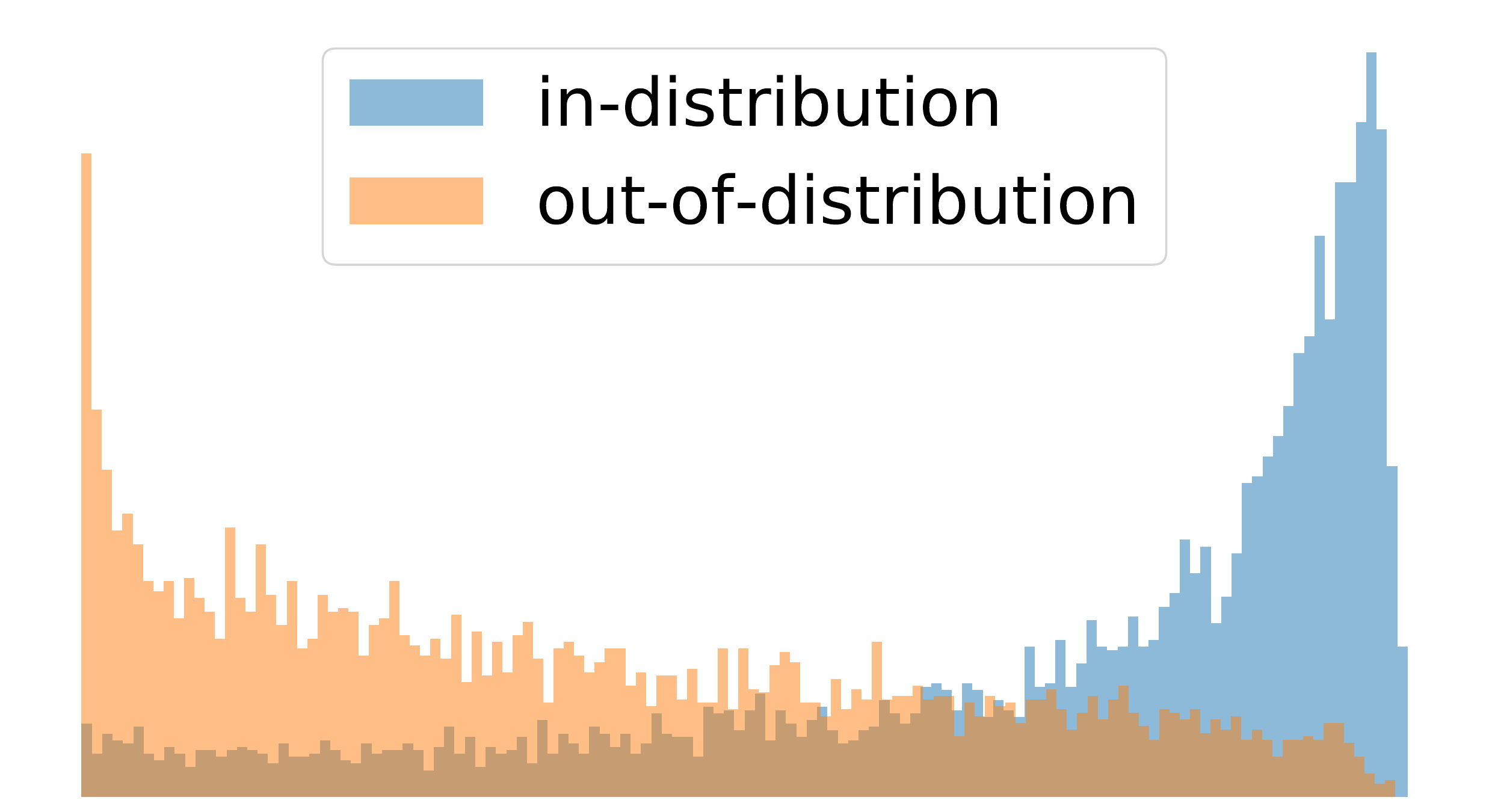}}
    \caption{Histograms of \texttt{LL} (left) and  \texttt{BVAE$_2$} (right) scores on the SVHN vs.\ CIFAR10 task, again showing that \texttt{BVAE$_2$} separates in-distribution and OoD data much more clearly than \texttt{LL}.}
    \label{fig:hist-svhn}
\end{figure*}

\subsection{Further Precision-Recall and ROC Curves}
We report precision-recall curves for all benchmarks, as well as ROC curves for the FashionMNIST (held-out classes) benchmark.
We show both types of precision-recall curves, depending on whether in-distribution data are considered as the false class (denoted by "in"), or whether OoD data are considered to be the false class (denoted by "out").
\cref{fig:fashion_examples} also shows examples from the FashionMNIST dataset, in order to help visualize the different class splits for the FashionMNIST (held-out class) benchmark.
\begin{figure*}[htb!]
	\centering
    \includegraphics[width=\textwidth]{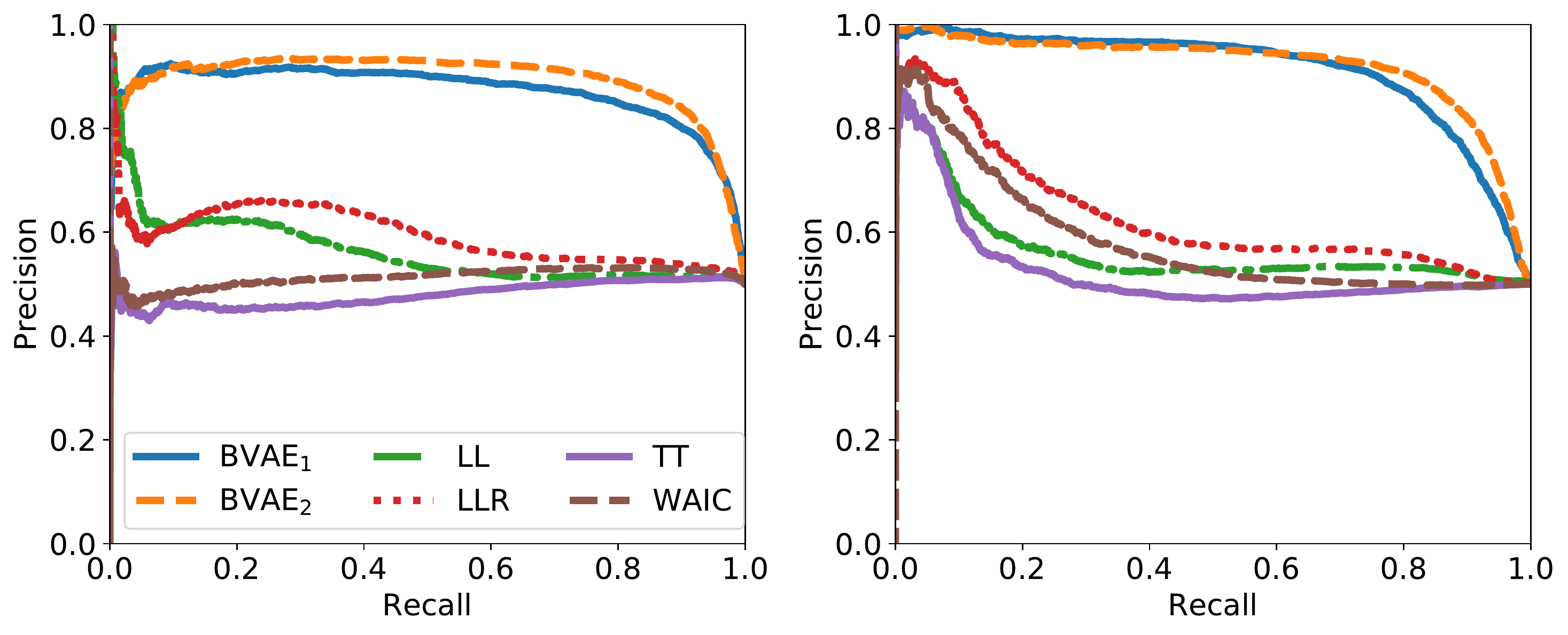}
    \caption{Precision-recall curves (in) and (out) on the FashionMNIST vs.\ MNIST benchmark.}
    \label{fig:fashion_examples}
\end{figure*}
\begin{figure*}[htb!]
	\centering
    \includegraphics[width=\textwidth]{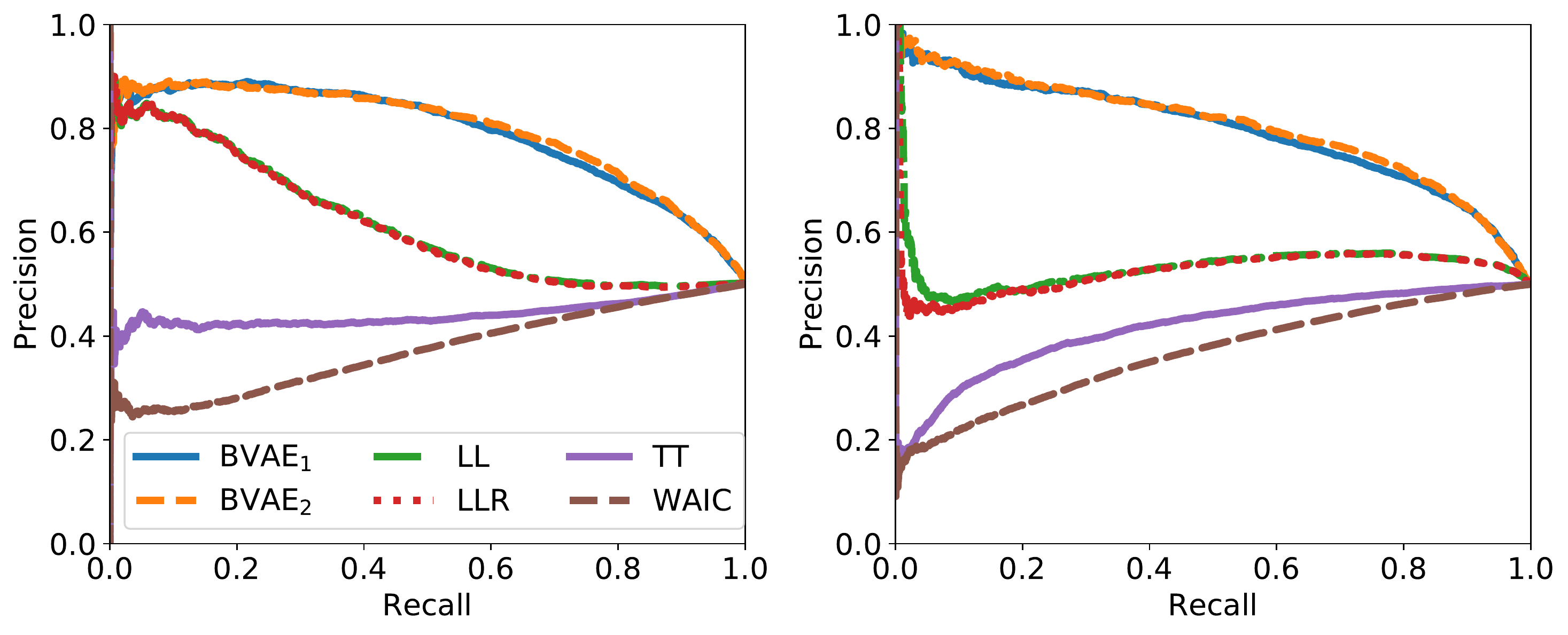}
    \caption{Precision-recall curves (in) and (out) on the SVHN vs.\ CIFAR10 benchmark.}
    \label{fig:prc}
\end{figure*}
\begin{figure*}[ht!]
    \centering
    \includegraphics{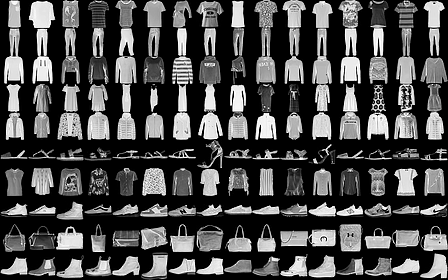}
    \caption{Examples from the FashionMNIST dataset for classes (from top to bottom) zero (t-shirt/top), one (trouser), two (pullover), three (dress), four (coat), five (sandal), six (shirt), seven (sneaker), eight (bag), and nine (ankle boot).}
\end{figure*}
\begin{figure*}[ht!]
    \includegraphics[trim={2.5mm 0 2.5mm 0},clip,width=\textwidth]{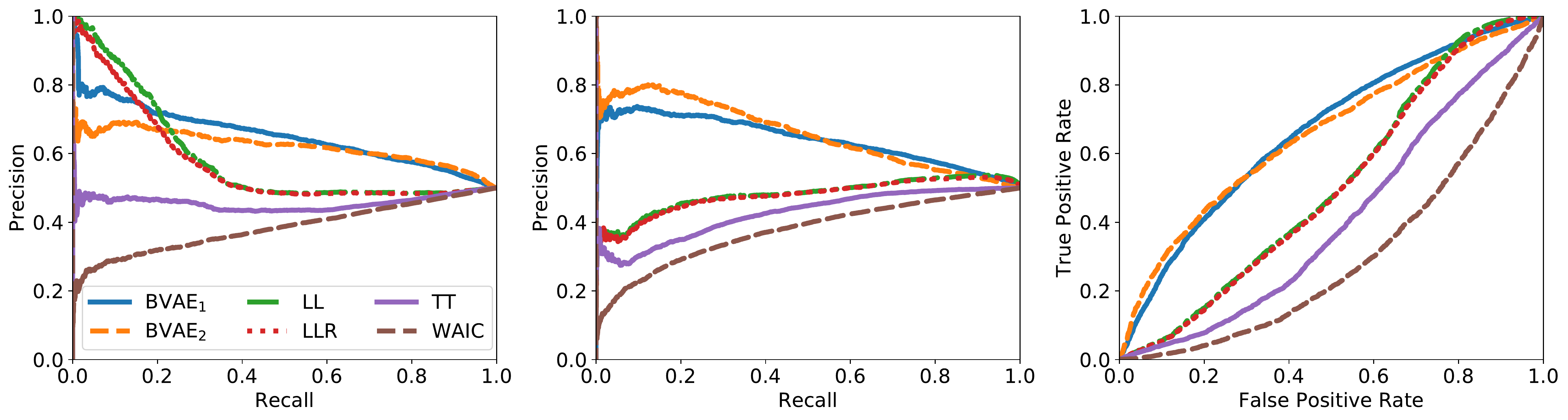}
    \caption{(Left) Precision-recall curves and (right) ROC curves of all methods on the FashionMNIST (held-out classes) benchmark with classes zero (t-shirt/top) and one (trouser) held-out.}
\end{figure*}
\begin{figure*}[ht!]
    \includegraphics[trim={2.5mm 0 2.5mm 0},clip,width=\textwidth]{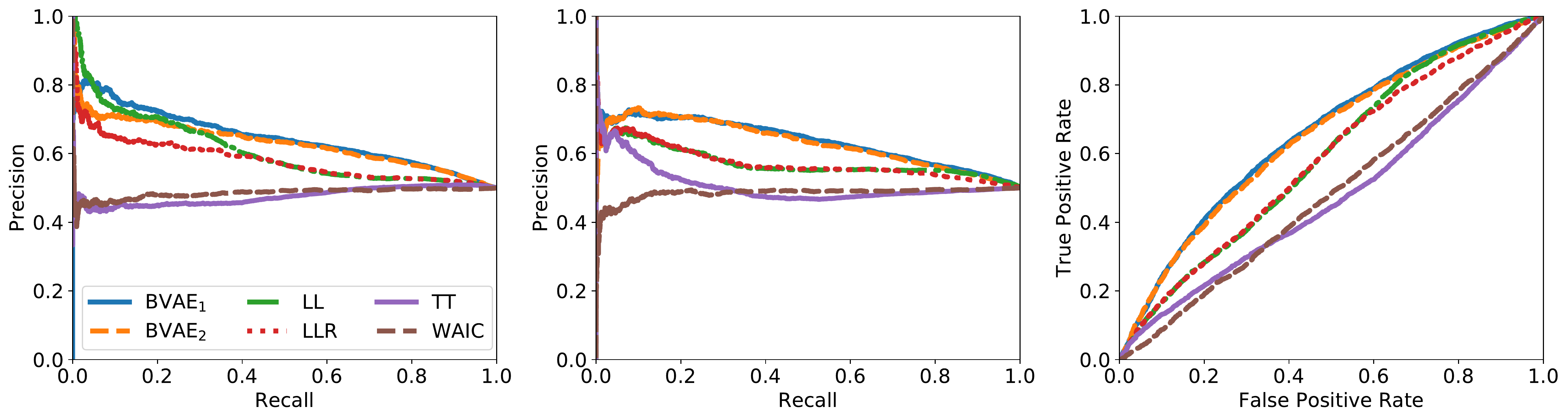}
    \caption{(From left to right) Precision-recall curves (in and out) and ROC curves of all methods on the FashionMNIST (held-out classes) benchmark with classes two (pullover) and three (dress) held-out.}
\end{figure*}
\begin{figure*}[ht!]
    \includegraphics[trim={2.5mm 0 2.5mm 0},clip,width=\textwidth]{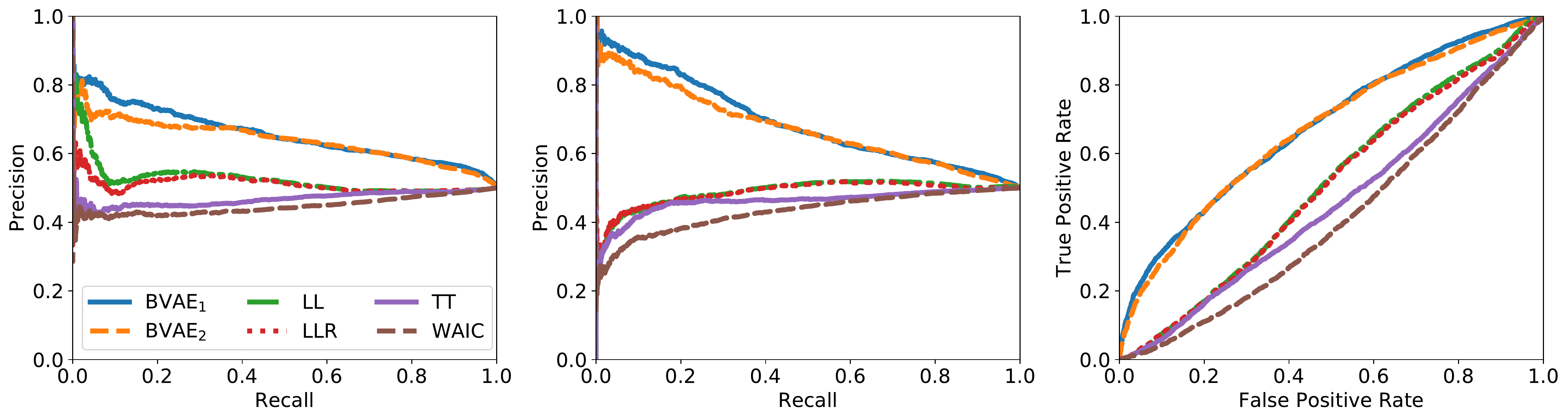}
    \caption{(From left to right) Precision-recall curves and ROC curves of all methods on the FashionMNIST (held-out classes) benchmark with classes four (coat) and five (sandal) held-out.}
\end{figure*}
\begin{figure*}[ht!]
    \includegraphics[trim={2.5mm 0 2.5mm 0},clip,width=\textwidth]{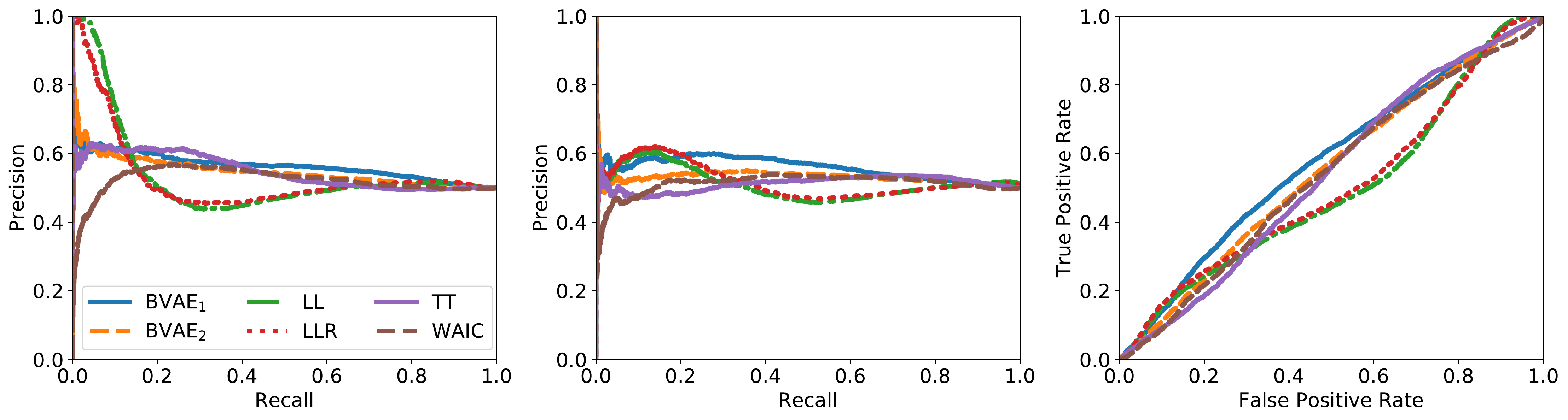}
    \caption{(From left to right) Precision-recall curves (in and out) and ROC curves of all methods on the FashionMNIST (held-out classes) benchmark with classes six (shirt) and seven (sneaker) held-out.}
\end{figure*}
\begin{figure*}[ht!]
    \includegraphics[trim={2.5mm 0 2.5mm 0},clip,width=\textwidth]{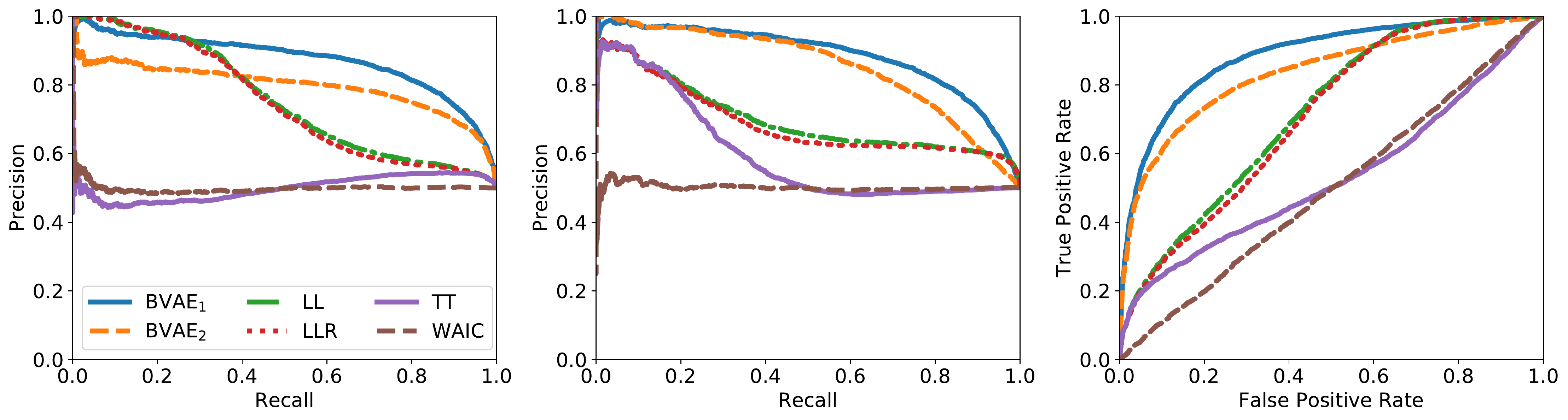}
    \caption{(From left to right) Precision-recall curves (in and out) and ROC curves of all methods on the FashionMNIST (held-out classes) benchmark with classes eight (bag) and nine (ankle boot) held-out.}
\end{figure*}

\end{document}